\documentclass[screen]{acmart}
\usepackage{endnotes}

\usepackage{dcolumn}
\newcolumntype{d}{D{.}{.}{-1}}
\newcommand{\MC}{\multicolumn}

\AtBeginDocument{%
  \providecommand\BibTeX{{%
    \normalfont B\kern-0.5em{\scshape i\kern-0.25em b}\kern-0.8em\TeX}}}

\setcopyright{rightsretained}
\copyrightyear{2022}
\acmYear{2022}
\acmDOI{10.1145/3531146.3533204}
\usepackage{subcaption}
\acmConference[FAccT '22]{ACM Conference on Fairness, Accountability and Transparency '22}{June 21--24, 2022}{Seoul, Republic of Korea}
\acmBooktitle{ACM Conference on Fairness, Accountability and Transparency,June 21--24, 2022, Seoul, Republic of Korea}
\acmISBN{978-1-4503-9352-2/22/06}

\newcommand{\withcomment}{0}
\newcommand{\hecom}{\if1\withcomment \color{blue} 
    \fi 
    \if0\withcomment \fi}

\begin{document}

\title[Algorithmic Fairness and Vertical Equity]{Algorithmic Fairness and Vertical Equity: \\ Income Fairness with IRS Tax Audit Models}




\author{Emily Black}
\authornote{Both authors contributed equally to this research. The views expressed in this paper do not necessarily represent the official position of the U.S. Department of the Treasury. 
This research was conducted using funding from the Institute for Human-Centered Artificial Intelligence (HAI) at Stanford and Arnold Ventures. The authors thank Edie Brashares, Bob Gillette, John Guyton, Tom Hertz, Barry Johnson, and Alex Turk for useful feedback and guidance and Brandon Anderson and Evelyn Smith for technical assistance.}
\email{emilybla@andrew.cmu.edu}
\affiliation{%
  \institution{Carnegie Mellon University}
  \country{USA}
 }

\author{Hadi Elzayn}
\authornotemark[1]
\affiliation{%
  \institution{Stanford University}
  \country{USA}}
\email{hselzayn@law.stanford.edu}

\author{Alexandra Chouldechova}
\authornote{Equal co-supervision}
\affiliation{%
  \institution{Carnegie Mellon University}
  \country{USA}
}

\author{Jacob Goldin}
\authornotemark[2]
\affiliation{%
  \institution{Stanford University and U.S. Treasury Department}
  \country{USA}
}

\author{Daniel E. Ho}
\authornotemark[2]
\affiliation{%
  \institution{Stanford University}
  \country{USA}
}
\newcommand{\w}{\mathbf{w}}

\renewcommand{\a}{\mathbf{a}}
\renewcommand{\w}{\mathbf{w}}

\renewcommand{\shortauthors}{Black and Elzayn, et al.}

\begin{abstract}
This study examines issues of algorithmic fairness in the context of systems that inform tax audit selection by the United States Internal Revenue Service (IRS).  While the field of algorithmic fairness has developed primarily around notions of treating like individuals alike, we instead explore the concept of \emph{vertical equity}---appropriately accounting for relevant differences across individuals---which is a central component of fairness in many public policy settings. 
Applied to the design of the U.S. individual income tax system, vertical equity relates to the fair allocation of tax and enforcement burdens across taxpayers of different income levels.  Through a unique collaboration with the Treasury Department and IRS, we use access to detailed, anonymized individual taxpayer microdata, risk-selected audits, and random audits from 2010-14 to study vertical equity in tax administration.  In particular, we assess how the adoption of modern machine learning methods for selecting taxpayer audits may affect vertical equity. Our paper makes four contributions. First, we show how the adoption of more flexible machine learning (classification) methods---as opposed to simpler models
---shapes vertical equity by shifting audit burdens from high to middle-income taxpayers. Second, given concerns about high audit rates of low-income taxpayers, we investigate how existing algorithmic fairness techniques would change the audit distribution. We find that such methods can mitigate some disparities across income buckets, but that these come at a steep cost to 
performance.
Third, we show that the choice of whether to treat risk of underreporting as a classification or regression problem is highly consequential.  Moving from a classification approach to a regression approach to predict the expected magnitude of underreporting shifts the audit burden substantially toward high income individuals, while increasing revenue.  Last, we investigate the role of differential audit cost in shaping the distribution of audits. Audits of lower income taxpayers, for instance, are typically conducted by mail and hence pose much lower cost to the IRS. We show that a narrow focus on return-on-investment can undermine vertical equity. Our results have implications for ongoing policy debates and the design of algorithmic tools across the public sector.  
\end{abstract}

\maketitle 

The annual tax gap, namely the difference between taxes owed and taxes paid, is estimated to be \$440B in the United States~\cite{tax_gap}.  Audits are the principal mechanism by which the Internal Revenue Service (IRS), the agency responsible for tax collection, verifies tax compliance and deters non-compliance.  IRS resources are limited and the agency must use audits judiciously.  During audits, the IRS typically solicits additional information from taxpayers to support information reported on filed returns.  For the taxpayer, audits can be time-consuming, stressful, and costly~\cite{macnabb,kiel2018s}. Low-income taxpayers, for whom refunds can comprise a substantial part of income, may wait ``on their refunds to pay day-to-day living expenses such as rent, car repairs, or healthcare, and any delay can cause taxpayers significant hardship"~\cite{taxpayer_adv}. 

Since the 1970s, the IRS has used classification models as part of its audit selection process to detect which individuals are most likely to have misreported their tax liability. While the use of both classical and modern machine-learning models is foundational to many government agencies' efforts to modernize predictive and allocative tasks \cite{engstrom2020government}, the adoption of such tools comes with considerable risks.  The algorithmic fairness literature has amply documented how disparate impact and other negative outcomes can arise from the uncritical adoption and application of such models~\cite{angwin2016machine, healthcare,dastin2018amazon}.   Given the scale and impact government decisions may have, mitigating these risks is a key priority for researchers and policy \cite{GAO, bideneo}.   In this work we study the impact of, and safeguards for,  fairness of machine learning models in the IRS tax audit context.

Specifically, our analysis focuses on fairness defined in terms of vertical equity, namely, appropriately accounting for relevant differences across individuals. This notion is central to public finance and public policy. By contrast, the algorithmic fairness literature has developed many formal definitions of fairness and techniques to satisfy notions of horizontal equity (treating like individuals alike) \cite{hardt2016equality,dwork2012fairness,kusner2017counterfactual}. The applicability of these techniques to improve vertical equity has been little-explored. More generally, the literature on how to apply algorithmic fairness techniques to improve real-world systems remains in a nascent stage, especially in high-stakes policy settings where direct data and systems access can be challenging.
Using anonymized IRS microdata, our work 
(i) examines the applicability of existing methods for promoting 
vertical equity in the tax audit context,
(ii) introduces new algorithmic fairness problems 
motivated by vertical equity considerations, and
(iii) provides a case study of addressing vertical equity concerns in a real-world algorithmic decision system.  By introducing vertical equity to algorithmic fairness, we follow in the footsteps of others \cite{hutchinson201950,heidari2019moral,binns2018fairness} that situate fairness in broader frameworks.


Our point of departure and the key motivation for our study is summed up in two key observations that, taken together, point to a discrepancy between the distribution of misreporting compared to the distribution of audits:
(1) the audit rate for lower-to-middle income earners is often as high or higher in recent recent years than that of high income earners; yet (2) an analysis of randomly conducted audits reveals that the amount of misreported tax liability (which we refer to, interchangeably, as the ``misreport amount'' or ``adjustment'') is highest among the highest income earners and the rate of misreporting---defined as misreporting above \$200---increases roughly monotonically with income.  
With this context, our key research questions are as follows: \\
    \indent (1) \textbf{To what extent does the choice of audit selection algorithm affect the noted discrepancy?} Given the discrepancy between ground truth misreporting and audit allocations, we might expect that introducing a more accurate model may mitigate the issue.  
    However, we observe empirically that more flexible models, while indeed increasing accuracy, have the effect of even \textit{further} concentrating of the audit burden on the lower-to-middle income taxpayers.\\
    \indent (2) \ \textbf{Can existing algorithmic fairness methods, originally designed to promote horizontal equity, be applied to improve vertical equity?}  
    In our context, one conception of  vertical equity consists of  monotonicity of the audit rate with respect to income. We show that, under some conditions, a selection process\footnote{By `selection process,' we mean the prediction model and 
    the process by which predictions are used to allocate audits together.} that satisfies the well-known fairness metrics of equal true positive rates and equalized odds also requires monotonicity of the audit rate with respect to the misreport rate. Given our empirical findings, this also implies monotonicity with respect to income. 
    We thus divide taxpayers into  \emph{income buckets}  
    and explore to what extent conventional fairness methods applied to such buckets can resolve the apparent discrepancy  between the audit rate and misreporting. We show that such methods come at a steep cost to revenue.  \\
    \indent (3) \textbf{What techniques can we use to more directly address vertical equity in the IRS audit allocation context?}
       We implement a direct approach to achieve 
       monotonicity by imposing allocation constraints on model outputs, and find that this approach results in a modest cost to revenue. However, we find that 
       switching the prediction task from classification to regression not only also achieves a roughly monotonic shape, closely matching the audit distribution of an \emph{oracle} with knowledge of the true misreport amount,
       but also obtains \emph{significantly more revenue} than even unconstrained classification. This is because regression shifts focus to taxpayers likely to have high amounts of underreporting rather than simply high probabilities of a misreport. \\
    \indent (4) \textbf{Can differential audit costs explain the status quo mismatch?} 
    We show that fully optimizing for return-on-investment with respect to the IRS' \emph{audit costs} concentrates audits nearly exclusively on lower income taxpayers, 
    even when using predictions arrived at via regression. This suggests that IRS budgetary constraints may play an important role in shaping the agency's ability to more equitably allocate audits without sacrificing the detection of under-reported taxes.  A narrow focus on return-on-investment can seriously undermine vertical equity goals.

A major contribution of this paper is that we conduct all our experiments on real, detailed, audit data collected by the IRS. We view this collaboration as an important case study to assess and mitigate disparities in real-world, public sector settings that operate subject to binding operational constraints  \cite[see][]{holstein2019improving,brown2019toward,lamba2021empirical,Obermeyer447,geyik2019fairness}. 
Our primary dataset consists of a stratified random sample of taxpayers collected as part of the IRS' National Research Program (NRP), allowing us to avoid the selective labels problem~\cite{lakkaraju2017selective}, to draw inferences on a representative dataset, and to directly measure the risk of misreporting. Our work also connects to work that emphasizes the choice of prediction task~\cite{Obermeyer447, mullainathan2021inequity} and problem formulation~\cite{passi2019problem} for algorithmic fairness. In addition, our results speak to current policy debates about the fairness of tax administration \cite{kiel} and appropriate funding levels for the IRS~\cite{national_taxpayer_advocate}. 

The paper proceeds as follows. Section~\ref{sec:background} provides background on the U.S. tax system and spells out the motivating stylized facts, setting up the question of what the IRS's turn to machine learning may portend for vertical equity. Section~\ref{sec:data} provides background on data and key definitions.  Section~\ref{sec:fairness}  formally describes the audit problem, introduces notation, and discusses how extant fairness metrics might apply to the IRS context. Our main investigation is presented in four parts. First,  Section~\ref{sec:moreflexible} examines the impact of more powerful classifiers on audit distribution. Second, Section~\ref{sec:fairlearn} presents the results of applying established algorithmic fairness techniques in our setting. Third, Section~\ref{sec:changing} studies the incorporation of monotonicity constraints as well as the simple but fundamental change of switching from classification to regression. Fourth, Section~\ref{sec:cost} examines the implications of accounting for audit costs.  Section~\ref{sec:discussion} concludes.


\vspace{-3mm}
\section{Background on the US Tax System\label{sec:background}}

We examine individual federal income taxes in the US system. Taxes are assessed based on self-reported liability statements called \emph{tax returns}, which can be time consuming and complicated to prepare; many taxpayers use commercial software or paid preparers.
The tax rate on income is progressive, with marginal tax rates increasing in income. 

As the tax code is very complicated, 
taxpayers {\hecom (and their preparers \cite{gaopreparers})} often make errors when calculating the amount they owe and are thus inadvertently non-compliant; others are willfully non-compliant, i.e., evade paying taxes.
The annual gross tax gap, which measures total noncompliance, is approximately \$440B~\cite{tax_gap}. In order to recover lost revenue, and to promote compliance with the income tax law, the IRS audits individuals that it believes may not be paying their full owed tax---due to, e.g., erroneously claiming credits or under-reporting income.

The IRS' audit selection system is complex, with many parts. It principally relies on: 
{\hecom (i) algorithmic methods to predict which taxpayers are most likely to underreport taxes, which serves as our main focus, (ii)  a combination of simple rules that flag returns automatically}; and, to a lesser extent, (iii) tips and other third party information, such as from whistleblowers.
We focus on the algorithmic component of the IRS audit selection process, which has historically been a classification algorithm predicting individual taxpayer misreport~\cite{hunter1996irs}. The details of existing modeling approaches are confidential, but historically, the basic approach involves a form of linear discriminant analysis. 

{\hecom
Audits are conducted in different ways depending on the size and scope of issues identified. Some audits, including most involving the Earned Income Tax Credit (EITC), are conducted by mail at relatively low cost to the IRS. More complicated and extensive audits may be conducted by interview or by IRS examiner field visits. The timing of an audit relative to the processing of a return also varies. For instance, audits may be conducted on taxpayers claiming refunds before a check is sent out; this is known as revenue protection, and such audits are called ``pre-refund''. Audits occurring after a check has been sent out to, or received from, the taxpayer are known as ``post-refund.'' These timing distinctions create differential impact on taxpayers, and may also affect the ease with which the IRS conducts audits.}

Over the last eight years, budget cuts have decreased the audit rate, from an overall rate of 1\% of individual filings receiving audits in 2010 to just 0.5\% in 2016~\cite{compliance_presence}. The audit rate has decreased most significantly for individuals earning between \$1-5M. Such individuals were audited at a rate of $\approx$8\% in 2010 but just 2.2\% in 2016~\cite{compliance_presence}. These changes in audit rates correspond to disproportionate reductions in examiners with more specialized expertise: while there was a 15\% reduction in examiners conducting correspondence audits (i.e. audits by mail) from 2010 to 2019, there was a 25-40\% reduction in examiners conducting in-person audits, which are utilized more for higher-income individuals~\cite{budget_dropoff}.

\begin{figure}[tbh]
   \centering
    \includegraphics[scale=0.225]{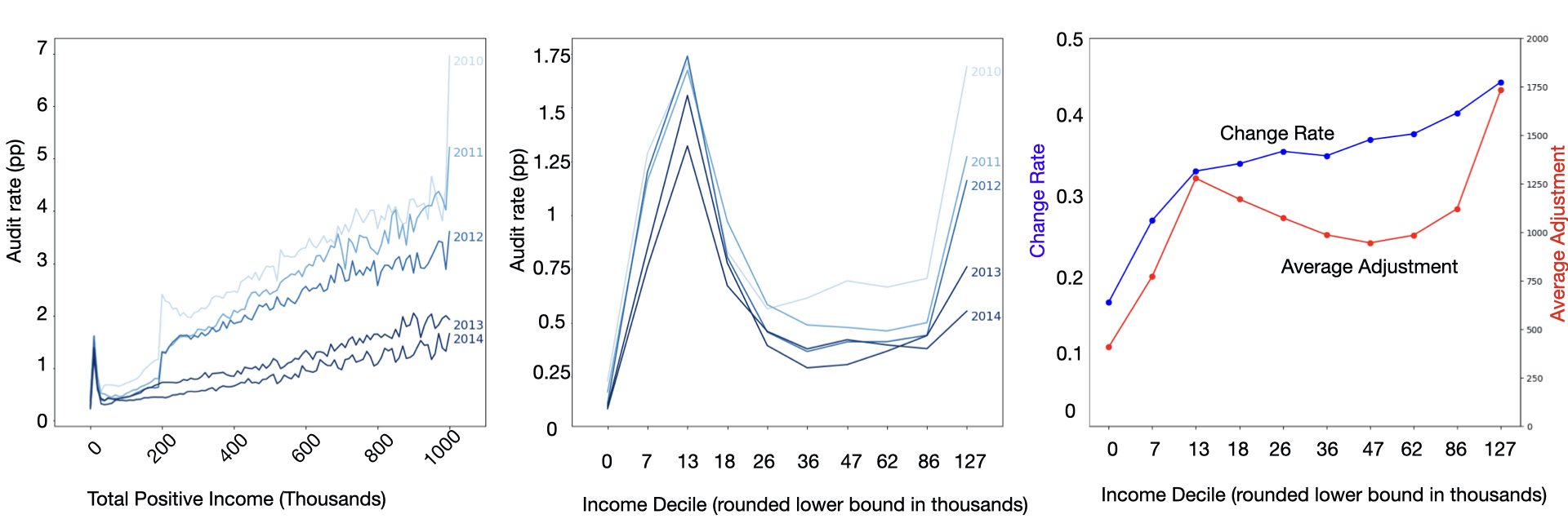}

    \label{fig:nrp-outcomes}
    \caption{Left two graphs: Audit Rate vs. Total Positive Income over time. Both of these graphs are calculated on operational audit (OP) data. Each line of a different color represents a different year, from 2010 to 2014. The x-axis indicates income binned into buckets of income, while the y-axis is the fraction of taxpayers in each bucket audited. On the leftmost, we have reported income buckets of \$10,000, up \$1m, while on the second graph, we show the same analysis over reported income deciles. Note that as the 10th income decile starts at 127K,  this graph is comparatively compressed.
    Right: Ground truth rates of misreporting (over \$200) (left) and average amount of misreporting conditional on misreport, aka average adjustment (over \$ 200), (right) over income. The results here are presented over  five years of NRP data 2010-2014, adjusted to 2014 dollars. The x-axis denotes income deciles, and the y-axis denotes rate of misreporting and average amount of misreporting in dollars, respectively. Taken together, we can see that there is a mismatch between audit allocation and ground truth noncompliance. 
    }
    \label{fig:motivation}
\end{figure} 

\vspace{-3mm}
\subsection{{\hecom Motivating Facts}}
\label{sec:puzzle}

We highlight two motivating facts relevant to our investigation. First, in the most recent years, the lowest income earners have been audited at the same rate as the highest income earners. The left panel of Figure~\ref{fig:motivation} plots income in \$10K bins from \$0 to \$1M
on the x-axis against the audit rate on the y-axis. Each line represents one year, from 2010 in lightest to 2014 in darkest blue. This panel shows the clear trend of the declining overall audit rate over time, which affects higher income groups most acutely. In addition, while audit rates generally increase in income, there is a large spike of audits in the lowest income groups. In 2014, the lowest earners are audited at a higher rate than all other income groups, except for those earning nearly \$1M. The middle panel depicts the same data using income deciles. After 2010, low-to-middle income taxpayers (i.e. those in the 2nd-4th income deciles from \$6.7K to \$26K), were audited at a higher rate than all higher income deciles. This reflects the particular focus on pre-refund audits done principally by mail. 

Second, the rate at which taxpayers understate their tax liability increases monotonically with income and average adjustments are highest in the highest income decile.  The right panel presents audit outcomes estimated on the NRP data (described in Section~\ref{sec:data} below).  The blue line in this panel depicts the estimated fraction of audits in each decile with a true misreport of at least \$200, while the red line depicts the average adjustment by decile.  Because this is a stratified random sample with corresponding sampling weights, it is free of the selection bias inherent in measuring outcomes among risk-selected  audits, and can thus be used to construct consistent estimates of \emph{population} non-compliance. 

These facts raise the motivating questions of this work: if adjustments are highest in the highest income decile, and the misreport rate increases monotonically with income, then why are audits so highly concentrated on lower-to-middle income taxpayers? To what extent can such patterns be exacerbated or mitigated by machine learning techniques? And are there opportunities for improving vertical equity given this mismatch?  


\vspace{-4mm}
\section{Data and Key Terminology}
\label{sec:data}

We address these questions through a unique collaboration with the Treasury Department and IRS, which provides us access to two data sets previously unexplored in the computer science literature: (1) the NRP data, which consists of line-by-line audits of a stratified random sample of the US population (n=71.9K ) from 2010-14~\cite{NRP}; and (2) all Operational Audits (OP) for 2014 (n=791.9K), which are risk-selected audits to identify tax evasion. Each observation contains information filed in a tax return. All dollar amounts are adjusted for inflation to 2014 dollars. 

We train and evaluate our machine learning models on NRP data, as this data is a random, representative sample of the US population and does not suffer from selection bias~\cite{lakkaraju2017selective}.\footnote{That is, when a return is selected for OP audit, the IRS has reason to believe that the return represents a misreport. Hence, the return is likelier to have a large adjustment than a randomly selected return from the population, and may be more generally non-representative as well. That said, one limitation is that prior work has found that NRP data under-reports higher income tax evasion~\cite{guyton2021tax}.}
We note that the OP audit data includes observations that were selected for audit not solely through machine learning tools, but also through rule-based flags such as internal inconsistencies, and other methods of selecting audits. We use the OP data to display the status quo of audit selection in the IRS as of 2014, for example, in the left-most graphs in Figure~\ref{fig:motivation}.

In this data, three concepts are particularly important. First, by \emph{income}, we mean the taxpayer's reported \emph{total positive income} (TPI). TPI captures all positive income an individual receives, gross of any losses.\footnote{Not all this income is taxable---for instance, tax deductions for losses or charitable contributions may reduce the total amount of taxable income.} We focus on reported (rather than audit-adjusted) income because that is what the IRS observes at the time it selects taxpayers for audit, and we focus on TPI (rather than taxable income) because it represents a simple measure of earnings that is less likely to be affected by audit determinations.  Many of the analyses in this paper will be over binned income, i.e. discretized income into equal-sized buckets, typically taken to be deciles of the income distribution.{\hecom \footnote{While these bins and associated thresholds are relevant to our analysis and implemented algorithms, to our knowledge they are not currently used by IRS to categorize returns or to determine taxpayer eligibility for benefits.}}  

Second, we refer to the amount by which a taxpayer's return understates true tax liability as the \emph{misreport amount}. If a taxpayer overstates their tax liability, then their misreport amount is negative. Throughout, we use the terms ``adjustment'' and misreport amount interchangeably. 
For classification, we define a \emph{significant misreport} as whether the taxpayer's understated tax liability exceeds a de minimis amount (\$200). For brevity, we refer to these simply as  \textit{misreports}.  Our findings are consistent across different choices of threshold (see Appendix~\ref{app:robustness_thresholds}).

Third, we define the \textit{cost} of an audit to the IRS as the total cost of the auditor's time recorded on the particular audit, which we compute from auditor time\footnote{Notably, our available data for auditor time does not account for auditor time spent on audit appeals. } and wage data. In principle, audit costs also include other components, such as overhead or attorney's fees for litigated cases, but these are not possible for us to measure with our data.  Note that we are focusing only on the budgetary costs of audits to the IRS, not the broader societal costs imposed on taxpayers.


\vspace{-3mm}
\section{The Audit Problem\label{sec:fairness}}

To explore vertical fairness in audit allocations, we start with the tools most readily available to improve the fairness of algorithmic tools: the now-canonical fairness definitions applied in the literature~\cite{hardt2016equality,verma2018fairness}. In this section, we first formalize the audit selection problem. Second, we discuss vertical equity in the context of the audit allocation problem, and consider how common fairness definitions may improve vertical as well as horizontal equity in this context.
Third, we discuss implementation of these metrics and model evaluation. 

\vspace{-3mm}
\subsection{Formal Definitions and Preliminaries}
\label{sec:prelims}
In this paper we define the basic audit problem as the following:
given a budget and a set of taxpayers with associated features and audit costs, return a selection of taxpayers for audit that detects and recovers as much under-reported tax liability as possible within the given budget.\footnote{In practice, the audit problem undertaken by the IRS must balance a variety of objectives, including revenue maximization,  deterrence,  minimization of taxpayer burden, and reduction of improper payments.}  

For the majority of this paper, we model the budget $K$ as a fixed number of audited tax returns, which we represent as a percentage of the population. We use a budget of $0.644\%$, which is the average percentage of audit coverage between 2010-2014.
Taxpayers are indexed by $i\in 1,...,N$ and have features $X_i$. One of the features in $X$ is $\mathcal{I}_i$, the taxpayer's income. The \emph{income bucket} $b_i\in \mathcal{B}=1...10$ of the taxpayer is the decile of $\mathcal{I}_i$.  Taxpayers submit a report of tax liability $\tilde{\ell}_i$, which may be different than their true liability $\ell_i$. We let $\delta_i = \ell_i -\tilde{\ell}_i$ denote the taxpayer's adjustment or misreport amount. We will also use $m_i=\mathbf{1}[\delta_i>\tau]$ for an indicator variable being above the misreport threshold $\tau$. In our main experiments, we set $\tau=200$, and write $\pi_i:= \Pr[\delta_i\geq \tau|X_i]$. We denote the cost incurred to the IRS by auditing an individual $i$ as $c_i$. We use $a_i$ as an indicator for whether taxpayer $i$ is audited, and $\alpha_i$ for a probabilistic relaxation. Occasionally, we use $\hat{\cdot}$ to indicate prediction, e.g. $\hat{\delta}_i$ as predicted misreport amount. 

The machine learning models we use throughout this paper which we integrate into the audit selection process either predict \emph{probability} of misreporting $\hat{\pi}_i$ (for classification models), or \emph{expected amount} of misreporting $\hat{\delta}_i$ (for regression models). In order to create an audit allocation from these predictions, however, we must select only $0.644\%$ of the population, which is in practice much less than the percentage of individuals predicted to not comply. Thus in order to create an audit allocation from machine learning model predictions, we rank model outputs by magnitude of prediction and take the top $0.644\%$. The audit problem can be formalized as: $\max_a \sum_{i} \delta_i\cdot a_i \text{such that    } \frac{\sum_i a_i}{N}<K$. 

If we consider $K$ to denote a \emph{dollar} budget as opposed to an audit rate budget, as we do in Section~\ref{sec:cost}, the constraint will be changed to $\sum_i a_ic_i<K$. In practice, we use $\hat{\delta_i}$ or $\hat{\pi}_i$ to approximate $\delta_i$.\footnote{As stated, this is an integer program, but we solve the linear relaxation due to computing constraints and because observations represent many people.}
\vspace{-3mm}
\subsection{Algorithmic Fairness and Vertical Equity}
 
We now discuss vertical equity in the IRS audit allocation context and its connection to several common algorithmic fairness metrics from the literature. 

{\hecom\textbf{Vertical Equity.} 
Vertical equity requires that different individuals be treated \emph{appropriately differently}. In the taxation and audit context, we focus on vertical equity with respect to the appropriate treatment of taxpayers at different income levels. Appropriately different treatment depends on context-specific considerations and value judgments.
To illustrate, given the fact that audits are costly for taxpayers (in terms of money as well as time, effort, and mental stress), policymakers may wish to avoid models that concentrate audits on low-income taxpayers out of concern for distributional social goals and in recognition of the declining marginal utility of taxpayers’ income. Other potential baselines for setting policy in this space are aligning audit rates with true rates of non-compliance, or with an \emph{Oracle}-based selection, i.e. an allocation which selects individuals in order of true misreport amount. In our setting, because under-reporting rates increase with income (Figure~\ref{fig:motivation}) and an oracle
 places a higher probability of selection as income increases, these factors would suggest that audit rates should increase in income as well. Motivated by such considerations, 
 we explore formalizing the notion of vertical equity as \textit{monotonicity}  and evaluate the discrepancy between audit allocation and true rates of misreport as an important component of vertical equity. Our focus on monotonicity is intended to  illustrate how one might incorporate vertical equity concerns into algorithmic fairness, but we note that a fuller analysis from an optimal tax framework is beyond our scope here.\footnote{A full optimal policy analysis would have to consider such factors as heterogeneity in the audit burden or in the deterrence effect of audits by income. For example, audits of higher income taxpayers can be more involved, but audits of lower-income taxpayers may require obtaining harder to produce information and often involve freezing refunds for liquidity-constrained taxpayers while the audit proceeds. A fuller optimal policy analysis would also need to consider how audit policies interact with other tax variables (such as the income tax schedule and underpayment penalties) for achieving revenue and distributional goals. Each of these factors may impact vertical equity.} 

\textbf{Montonicity} Monotonicity (with respect to income) would require that the audit probability increase as income increases. Formally, given income buckets $b$ and $b'$, $b \geq b' \implies \Pr[a_i=1|b_i=b] \geq \Pr[a_i=1|b_i=b']$.
We consider directly constraining the audit allocation to be monotonic in Section~\ref{sec:mono}.

\textbf{Oracle Allocation}
 An \emph{oracle} is a theoretical omniscient model with access to the true amounts of misreporting in the data (i.e. the ground truth labels).  Formally, the oracle represents the model $\hat{\delta_i}=\delta_i$, where $\delta_i$ is the amount of true misreport of individual $i$. The oracle creates an audit allocation by selecting individuals for audit in order of their true amount of misreport amount until exhausting the allocation budget. 
Thus, the 
audit allocation selected by the oracle 
is naturally aligned with true incidence of misreport. 
Although we do not explicitly enforce this behavior, we evaluate the vertical equity of model allocations by the extent to which they match the audit rate by income of the oracle model.}

\textbf{Demographic Parity.} Demographic Parity (DP) requires, in our context, equal audit probability across income buckets. That is: $\Pr[a_i=1|b_i=b] = \Pr[a_i=1|b_i=b'], \forall \ b, \ b'.$
Note that with a fixed budget and groups of equal size, asking for DP amounts to requiring the same audit rate for each group, which weakly satisfies monotonicity.
Compared to the status quo described in Figure~\ref{fig:motivation}, this would result in lower audit rates for low-to-middle income taxpayers as well as very high income taxpayers, and higher audit rates for middle-to-upper income taxpayers. Important limitations to DP include that (1) as noted, equal audit rates do not imply equal audit burdens if taxpayers bear different costs, and (2) a perfectly accurate classifier would not satisfy DP unless the misreporting rates are exactly equal, which they are not.

 \textbf{Equal True Positive Rates ~\cite{hardt2016equality}.} Equal  True Positive Rates (TPR) requires that the audit probability of \emph{non-compliant} taxpayers not depend on income group, i.e.,  $\Pr[a_i=1|m_i=1,b_i=b]=\Pr[a_i=1|m_i=1,b_i=b'],  \forall \ b, \ b'$.
Equal TPR ensures that no group of non-compliant taxpayers can expect a higher or lower chance of audit based solely on their income, but this does not mean that compliant taxpayers of each income group face the same chance of an audit.

\textbf{Equalized Odds.} Equalized Odds (EO) asks that the audit probability of both compliant and non-compliant taxpayers should not depend on their income group, i.e.: $\Pr[a_i=1|m_i=0,b_i=b] =\Pr[a_i=1|m_i=0,b_i=b']$, and  $\Pr[a_i=1|m_i=1,b_i=b] =\Pr[a_i=1|m_i=1,b_i=b']$.
EO extends equal TPR fairness by requiring audits of  compliant taxpayers at the same rate across groups in addition to auditing non-compliant taxpayers at the same rate across groups. 

In Appendix A, we consider  conditions under which equal TPR or EO will result in monotonicity of the audit rate with respect to income. Specifically, we consider a hypothetical allocation that audits all taxpayers with $\hat{\pi}_i>0.5$, and show that under certain (differing) conditions, audit allocations that satisfy either either equal TPR or EO will result in monotonicity of the audit rate with respect to the misreport rate. Because the misreport rate increases with income (Figure~\ref{fig:motivation}), this suggests that enforcing one of the fairness constraints on a model generating audit allocations may also lead to monotonicity of audits with respect to income. We note that this result is suggestive, since models that satisfy a fairness constraint for the hypothetical allocation described above need not do so for the actual audit allocation induced after imposing a budget. Thus, we must ultimately test whether the targeted fairness constraints are satisfied on the audit allocation that results from a model once a budget is incorporated. Next, we describe algorithms to instantiate these conditions and evaluate the performance tradeoffs. We implement these algorithms and report results in Section \ref{sec:fairlearn}.

\vspace{-3mm}
{\hecom\subsection{Model Evaluation}}
In order to compare model allocations, we will consider several performance metrics. {\hecom First, in order to approximate how well an audit allocation matches the ground truth rate of misreport, we consider how closely audit rates correspond to selection based on an oracle. 
Specifically, we calculate the \emph{overlap} between a model's allocation and the oracle's, formally, the size of the intersection of the model and oracle's audit allocation over the total number of audits in an allocation: $\frac{\sum_i a_{i,O}a_{i,M}}{K \times N}$, where $a_{i,O}$ and $a_{i,M}$ represent audit indicators for the oracle and a model respectively, $K$ is the audit budget as a percentage of the population, and $N$ is the total number of taxpayers.\footnote{The total number of taxpayers, taking into account the sampling weights. This metric is equivalent to the top-k intersection of model outputs, where $k$ is the audit allocation budget. This metric is often used to compare model-generated explanations~\cite{ghorbani2019interpretation,dombrowski2019explanations,black2021selective}.} Note that the overlap will be between 0 and 1, with 1 representing an exact match of the oracle's allocation.  We consider models that more closely match the oracle allocation with respect to income to have preferable vertical equity performance in our context.}

Second, we consider \emph{revenue} collected, which is simply the sum of adjustments over all audits. Recovering revenue is one of the key goals of the IRS and is itself relevant for distributive policy, since it funds services provided to citizens.  We define revenue as follows: $\sum_i \a_i \delta_i$.\footnote{We take sampling weights into account in this calculation, so in practice we calculate revenue as $\sum_{i \in |D|} \a_i \w_i \delta_i$, where $|D|$ is the size of the NRP data set, and $\w_i$ is the sample weight assigned to each row.} 

Third, we consider the \emph{no-change rate}, which is the fraction of audits resulting in no (substantial) adjustment. No-change audits are undesirable from both IRS and taxpayer perspective, as both the auditor and taxpayer could have saved significant time, effort, and stress. We define the no-change rate as $\frac{\sum_i \a_i \cdot (1-m_i)}{\sum_i \a_i}$. 

Fourth, we consider the \emph{cost} of the audit to the IRS, which is important both in terms of the feasibility of an audit policy and its net revenue implications. We define cost as  $\sum_i \a_i c_i$, where $c_i$ is our estimate of cost per return.\footnote{Similarly to revenue, in practice, we calculate cost as $\sum_{i \in |D|} \a_i \w_i c_i$.} 
We describe how we obtain cost estimates in Section \ref{sec:cost}. In Sections \ref{sec:moreflexible}-\ref{sec:changing}, we hold audit rates fixed and measure incurred cost. In Section \ref{sec:cost}, however, we consider constraints on the total dollar cost of policies, and show how they may help explain the existing discrepancy between income and the audit rate.


\subsection{Model Implementation}
There exists a large body of research surrounding how to best implement and guarantee the common fairness metrics outlined above~\cite{agarwal2018reductions, kearns2018preventing, donini2018empirical,celis2019classification, zafar2019fairness, hardt2016equality}. From this rich literature, we choose to rely on a technique developed by Agarwal et al.~\cite{agarwal2018reductions}, which intervenes in a model's training process to add a constraint during optimization which incentivizes the model to satisfy a given constraint in its predictions~\cite{agarwal2018reductions, donini2018empirical}. 
Methods that enforce fairness constraints during training time are often described as ``in-processing," as opposed to those which intervene at prediction time, which are called ``post-processing."
Agarwal et al.'s (in-processing) technique allows for demographic parity, true positive rate parity, equalized odds, and other constraints to be satisfied in expectation in a model's predictions on the training distribution. We include results from other methods of enforcing fairness constraints, including post-processing techniques, as a discussion of the differences between various methods in Appendix \ref{app:more_fairness}.


\vspace{-3mm}
\section{Flexible classifiers and audit classification \label{sec:moreflexible}}

\begin{table}[t]
\centering
\small
\begin{tabular}{llc|ddddd}
\MC{1}{c}{Model Type} & \MC{1}{c}{Label} & \MC{1}{c|}{Fairness} & \MC{1}{c}{Revenue}  & \MC{1}{c}{No-Change} & \MC{1}{c}{Cost} & \MC{1}{c}{Net Revenue} & \MC{1}{c}{Oracle}\\
 & \MC{1}{c}{Type} & \MC{1}{c|}{Constraint} &  \MC{1}{c}{(\$B)}  & \MC{1}{c}{Rate} &  \MC{1}{c}{(\$B)} & \MC{1}{c}{(\$B)} & \MC{1}{c}{Overlap} \\
\hline
Oracle & - & \texttimes & 29.40 & 0.0\%  &  0.33 & 29.07 & 1.00\\
LDA  & Class & \texttimes & 6.07 & 12.8\% & 0.21 & 5.86 & 0.09\\
Random Forest & Class & \texttimes &  3.05 & 3.5\% & 0.08 & 2.97 & 0.00\\
Grad Boosted & Class & \texttimes & 4.05 & 4.2\% & 0.08 & 3.97 & 0.00\\
Random Forest & Class & \checkmark (DP) & 2.75 & 8.0\% & 0.07  & 2.67 &  0.08\\
Random Forest & Class & \checkmark (TPR) & 0.69  & 12.4\% & 0.15 & 0.54 & 0.04\\
Random Forest & Class & \checkmark (EO) & 0.53  & 13.6\% &0.15 & 0.38 & 0.04\\
Random Forest & Class & \checkmark (Mono) & 3.00 & 4.0\% & 0.10 & 2.90  &  0.01 \\
Random Forest & Reg & \texttimes  & 10.22  & 23.3\% & 0.50 & 9.72 & 0.23\\
Grad Boost & Reg & \texttimes & 10.20  & 20.0\%  & 0.50 & 9.70 & 0.22\\
%
%
\end{tabular}
\caption{Revenue, no-change rate, cost,  net revenue, and oracle overlap for all models considered in this paper. No-change rate represents the percentage of audits that were allocated to compliant taxpayers; cost reflects cost to the IRS as described in Section~\ref{sec:cost}. These results reflect audit allocations that select the top $0.644$\% of taxpayers predicted most likely to misreport from each model. All metrics are reported on the test set, using the representative NRP sampling weights to scale up to the US taxpayer population. \label{table:all_nums}}
\end{table}

We begin by examining the hypothesis that the disproportionately high audit rate observed for low income earners may stem from using simpler classification models in guiding audit allocations.   
We demonstrate that (i) the disparity displayed in audit rates does not appear to arise from the less complex models similar to those the IRS has historically used; and, (ii) applying more complex models---in this case, Random Forests and Gradient Boosting--- actually \textit{exacerbates} the burden on lower income taxpayers. 

\vspace{-3mm}
\subsection{Experimental Setup} 
In this section, we consider the audit allocation determined by Linear Discriminant Analysis (LDA) (an approximation of the historical choice by the IRS), a Random Forest Classifier, and a Gradient Boosting Classifier. In principle, classifiers may perform well at reducing the no-change rate, furthering IRS's objective to avoid burdening compliant taxpayers. To be clear, the audit allocation 
is not simply the model's predictions, but rather the individuals most highly predicted for misreport up to the audit budget, as described in Section~\ref{sec:prelims}. We use NRP data from 2010-2014 to train all models in this paper to predict the likelihood of misreporting. We randomly split this data into a train and validation (75\%) and test (25\%) sets. We search for optimal hyperparameters using \textit{sklearn}'s GridSearchCV method with 5-fold cross validation.\footnote{As described in detail in Appendix~\ref{app:exp_details}, we train all but LDA models with \emph{sampling weights} provided in the NRP data, meant to ensure the data is representative of the taxpayer population. 
For LDA models, we sub-sample a dataset from the NRP data that respects the sample weights by randomly selecting (with replacement) rows from the weighted training data according to the weights. For example, suppose that each row $x$ has a sample weight $w$, and the sum of all weights in the training set is $W$. Then each observation has a $\frac{w}{W}$ chance of getting selected as any given row in the sub-sampled data.}

All results in this and following sections are calculated on the test set, which is  
reserved for reporting results. 
Results are reported by rescaling costs and revenues to reflect estimated average annual values for the full population (averaged between 2010-2014). For each classification model, we sort taxpayers in descending order of predicted \emph{misreport probability} to produce a ranking. 
We then apply an audit rate budget of $0.644$\% of the population, reflecting the average audit rate from 2010-2014,  and select audits $a_i$ by taking the top $0.644$\% of the population (i.e. 1125000 audits)in rank order. 
Further 
details are in  Appendix~\ref{app:exp_details}.

\begin{figure}[t]
    \centering
    \includegraphics[scale=0.225]{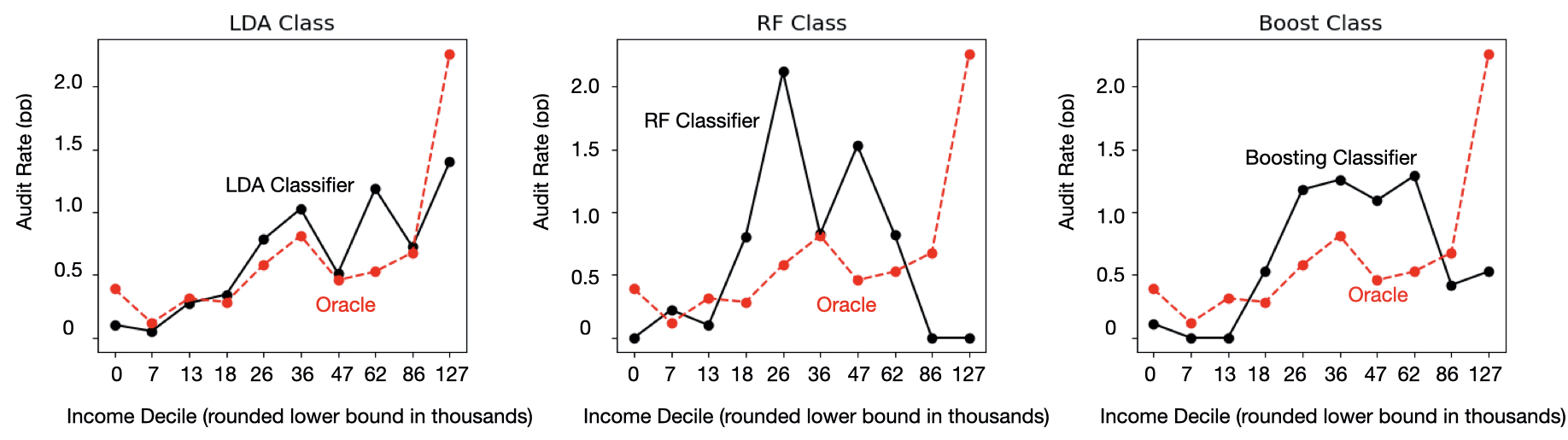}
    \vspace{-0.15in}
    \caption{From Left to right: Audit Rate by Income LDA Classifier, Random Forest Classifier, and Gradient Boosted Classifier, presented in black. The oracle allocation on the same budget is presented in red on the same graph.}
    \label{fig:class_models}
    \vspace{-0.2in}
\end{figure}

\vspace{-3mm}
\subsection{Results} 
Figure~\ref{fig:class_models} displays the audit rate by income of allocations obtained via 
ranking the predictions of LDA, Random Forest Classification, and Gradient Boosted models by predicted probability of misreport and selecting the top $0.644$\% of the population. Revenue and no-change rate of these models are included in Table~\ref{table:all_nums}. We highlight implications below. 

First, higher model flexibility can lead to high audit focus on lower and middle income populations. As Table~\ref{table:all_nums} shows, the Random Forest Classifier is well-optimized for the classification task: it has an extremely low no-change rate---just 3.5\%---whereas simpler models have no-change rates higher than 12.8\%. However, the Random Forest Classifier focuses almost exclusively on the lower-middle and middle of the income spectrum, not targeting the highest earning 20\% at all. Similarly, the Gradient Boosted classification model concentrates most of the audit selection to the middle of the income spectrum (4-8th decile), with a strong drop-off for the top 20\% of the population. 
(Appendix~\ref{app:log_reg} shows that another simpler model (logistic regression) also results in rough monotonicity.)

Second, the simpler LDA model more closely matches the oracle. The LDA classifier has an audit selection curve that is roughly monotonic in income, with large increases in audit rate in the high income region. As 
LDA has been the IRS' historical modeling approach (although it differs in practice with our implementation), this suggests that the large spike in operational audit selection rate on the lower end of the income spectrum apparent in 2014 may not stem directly from the predictions algorithmic components of the decision system, but rather other policy and modeling choices.  

Third, increased classification accuracy does \textit{not} imply increased revenue. 
Table~\ref{table:all_nums} shows that the Random Forest and Gradient Boosted models have significantly lower no-change rates than the LDA model (3.5\% and 4.2\% vs. 12.8\%), yet also substantially \emph{lower} revenue ($\approx$\$3B and \$4B vs.\ $\approx$6B). This highlights that improved performance on one objective (e.g., accuracy) may come at the expense of other seemingly intertwined objectives (e.g., revenue).


\vspace{-3mm}
\section{Fairness Constrained Classification\label{sec:fairlearn}}
We now explore the use of bias mitigation methods to promote vertical equity.

\vspace{-3mm}
\subsection{Experimental Details}
We enforce algorithmic fairness definitions on the Random Forest model at different points in the audit selection process: \emph{during} training, or in-processing, following Agarwal et al. ~\citep{agarwal2018reductions}, and \emph{after} training but before prediction, or post-processing (deferred to Appendix \ref{app:more_fairness}, following Hardt et al.~\cite{hardt2016equality}). 
Our setup for training the fairness-constrained models mirrors our setup for the fairness-unconstrained models, with the exception that
we do not train the models with sampling weights, but rather subsample a dataset from the NRP weighted data as we do for LDA models as described in Section~\ref{sec:moreflexible}. 
This is 
because the in-processing methods are implemented using the FairLearn package~\cite{bird2020fairlearn}, and the FairLearn package leverages \textit{sklearn}'s sampling weight functionality in the course of their algorithm.

\vspace{-3mm}
\subsection{Results}
\begin{figure}
    \centering
    \includegraphics[scale=0.225]{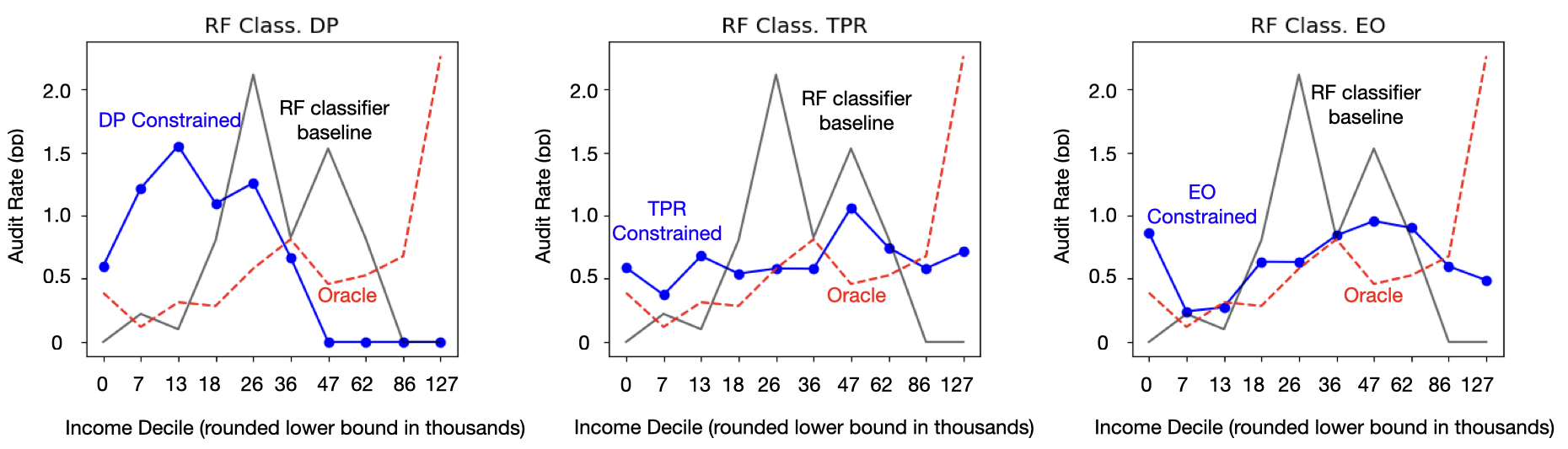}
    \caption{In-process fairness techniques imposed on a Random Forest model. From left to right: enforcing Demographic Parity (DP),   Equal True Positive Rates (TPR), and Equalized Odds (EO). Black (blue) series represent the unconstrained (constrained) allocation.}
    \label{fig:inprocess}
\end{figure}
Our high-level result is that enforcing fairness constraints during training results in steep trade-offs with 
limited fairness payoffs for the budgeted allocation problem. 
Figure~\ref{fig:inprocess} displays audit rate by income decile for Random Forest Classifier trained to respect each of the fairness definitions considered. We present revenue and no-change rate in Table~\ref{table:all_nums}. 

Equal TPR and EO models do lead to overall lower focus on low and middle income groups. However, they continue to under-target the highest end of the income spectrum when compared with the oracle predictor.  And perhaps surprisingly, despite this shift to focus slightly more on higher ends of the income spectrum, enforcing these constraints actually leads to a large decrease in revenue: from over \$3B to as low as \$600M in revenue. We additionally notice a decrease in the no-change rate towards levels closer to the baseline LDA predictor.  Finally, they imperfectly enforce the targeted fairness constraints once the audit budget is imposed: this is immediately evident in the allocation from a model constrained to respect demographic parity, as the audit rate is not equal across groups. 

Given these results, we argue that enforcing fairness constraints during training is not an effective technique to improve vertical equity in an audit allocation setting. We highlight some broader implications of vertical equity for algorithmic fairness in Section~\ref{sec:discussion}.


\vspace{-4mm}
 \section{Enforcing Monotonicity \label{sec:mono}}
    

In this section, we instead enforce monotonicity directly. We do this by solving the following linear program:
\small
\begin{align*}
    \max_{\alpha} \sum_{b \in \mathcal{B}} \sum_{i \in b} \alpha_i \hat{\pi}_i w_i \qquad  \text{ s.t. } &\alpha_i \in [0,1] \forall i;
    &\sum_{b \in \mathcal{B}} \sum_{i \in b} w_i \alpha_i = 1 ;
    &\sum_{i \in b_1} \alpha_i w_i \leq \sum_{i \in b_2} \alpha_i w_i &
    &\dotsb 
    &\sum_{i \in b_9} \alpha_i w_i \leq \sum_{i \in b_{10}} \alpha_i w_i
\end{align*}
\normalsize
where all notation follows Section \ref{sec:prelims}, $w_i$ represents sampling weights, and the Random Forest Classifier generates $\hat{\pi}_i$. 

The leftmost panel of Figure \ref{fig:reg_models} shows the audit distribution of the solution to the linear program. Notably, all income buckets from the fourth decile and above are audited at the same rate.  In other words, the constrained solution audits higher income deciles at the minimum in order to focus most energy on the fourth decile. The trade-off with performance is relatively modest relative to the unconstrained classifier, as seen in Table \ref{table:all_nums}: revenue does decrease, but by only \$50 million; the no-change rate increases by half a percentage point. These results indicate that, especially compared to enforcing traditional fairness constraints, enforcing monotonicity may be a relatively economical approach to encourage (one notion of) vertical equity. The next section shows, however, that this approach may be far from optimal. 

\vspace{-3mm}
\section{From Classification to Regression \label{sec:changing}}

We now demonstrate that changing the model's prediction target from the \emph{probability} of misreport to \emph{expected misreport amount}---i.e. changing from a classification to regression algorithm--- can reduce burden on lower-income taxpayers and make audit rates more closely mirror the oracle while also \emph{increasing} revenue. 
This demonstrates that, in some circumstances, changing the model's prediction task to reflect behavioral desiderata--rather than enforcing a constraint on top of a model optimizing for an imperfectly aligned task---is a more effective technique to reach equity goals.

\begin{figure}
    \centering
\includegraphics[width=\textwidth]{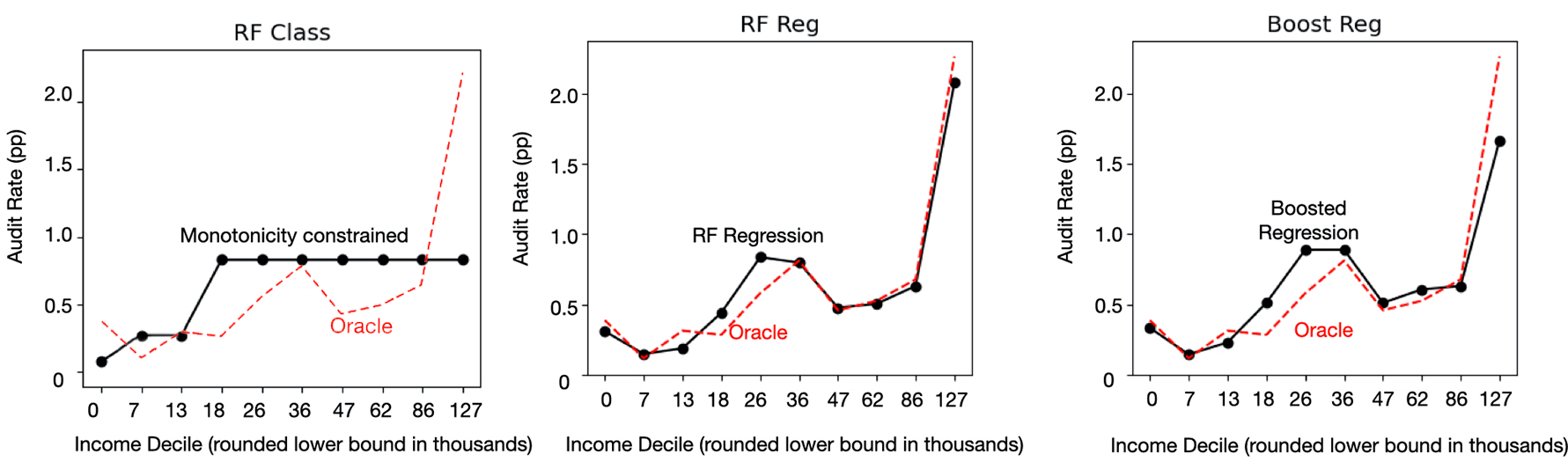}
    \vspace{-0.2in}
    \caption{Left: Monotonicity constraints explicitly enforced on audit allocations of a Random Forest Classifier. The black line represents the allocation, the red line represents the oracle. 
    Right: Audit Rate by Income in  Random Forest Regressor, and Gradient Boosted Regressor, presented in black. The oracle allocation on the same budget is presented in red on the same graph.}
    \label{fig:reg_models}
\end{figure}
We train regression models with the same process described in Section~\ref{sec:moreflexible} for classification models, 
but use the misreport amount as the label rather than to a binary indicator of misreport.
The audit rate by income decile of Random Forest and Gradient Boosting regression models are displayed in black in Figure \ref{fig:reg_models}, along with the oracle in dashed red. 

We highlight two chief results. First, shifting the prediction target from the probability of misreport (classification) to the expected amount of misreport (regression) 
{\hecom shifts audit focus from lower income to higher income taxpayers, resulting an audit allocation that is not only nearly monotonic, but also closely matches the oracle allocation.  As can be seen in Figure \ref{fig:reg_models} and the right column of Table~\ref{table:all_nums}, the resulting allocation is in fact closer to the oracle than any other prior allocation.}
Thus, changing from a classification to a regression task can be seen as one method to directly optimize for {\hecom (multiple notions of) vertical equity} 
in the IRS context. 

{\hecom Second, while changing the prediction target from presence of significant misreport to amount of misreport does increase the no-change rate (up to 20-23\%), it also results in a dramatic increase in revenue. Table~\ref{table:all_nums} shows that assessed revenue under regression rises to \$10B, compared to the \$3.6B baseline of high-powered classification models. } 

{\hecom Thus,} within the set of higher complexity models, switching from classification to regression may provide an effective way to decrease the mismatch between audit allocations and ground truth levels of misreport, as well as decrease audit focus on lower and middle income individuals, while \emph{increasing} under-reported tax liability detected by the IRS. We leave the discussion of how regression-based allocations interact with the IRS goal of broad-spectrum noncompliance deterrence---which may necessitate additional focus on lower-magnitude noncompliance---to future work. 


\vspace{-3mm}
\section{Agency Resources and The Impact of a Narrow Return-on-Investment Approach\label{sec:cost}}
\begin{figure}
    \centering
        \includegraphics[scale=0.25]{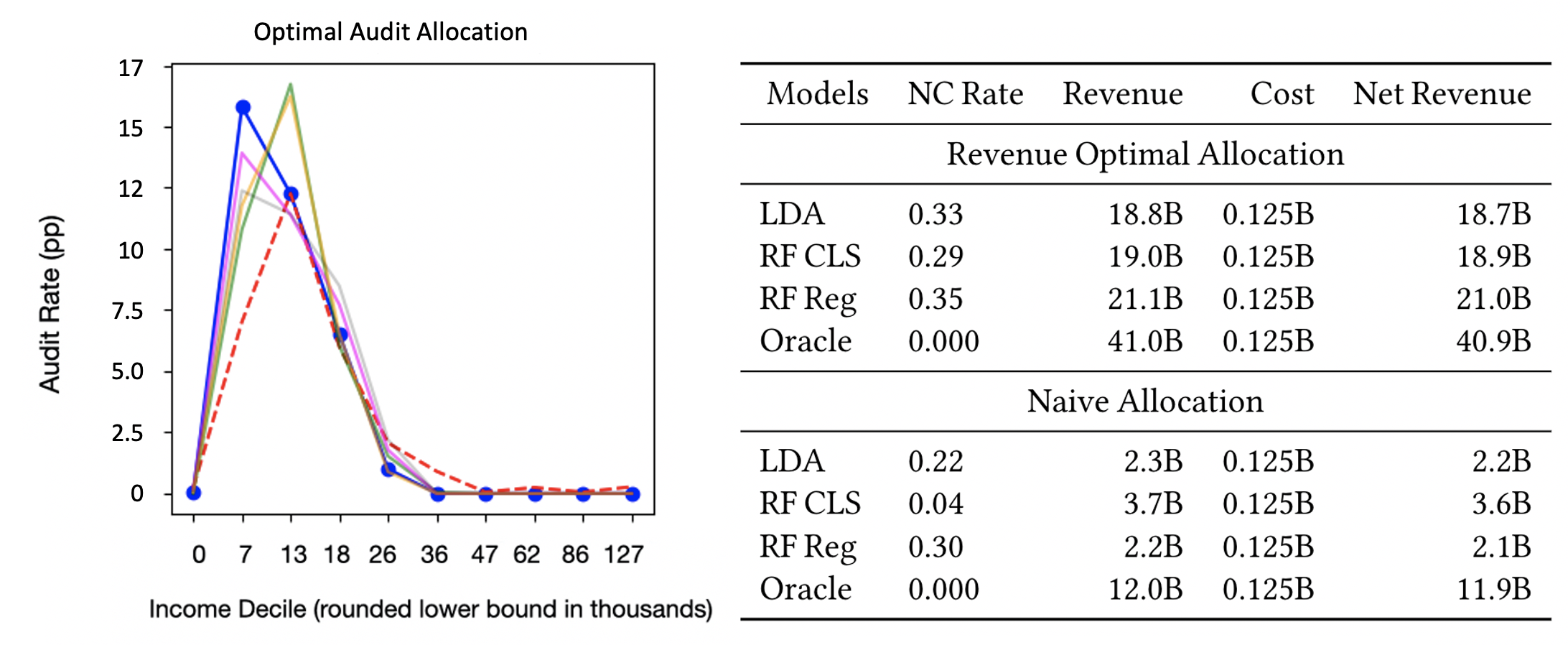}

    \vspace{-0.15in}
    \caption{Left: Revenue-optimal allocation from all models considered in paper so far, considering budget as a dollar amount. The x-axis represents income deciles, and the y-axis represents audit rate. We consider the budget to be 125 million, or the average budget over 2010-2014 using our approximation of cost described in Section~\ref{sec:cost}. The revenue-optimal allocation requires that the individuals with the highest \emph{ratio of revenue returned to IRS over cost to the IRS} are selected for audit up to the dollar budget, which results in a similar allocation from all models. 
    Right: No-change rate, revenue, cost, and net revenue of allocations from different models considered in the paper when modeling audit budget as a dollar amount, for both for net-revenue optimal and naive allocations.  \label{fig:cost_constrained}}
   \vspace{-0.2in}
\end{figure} 

{\hecom We now turn to examining the relationship between vertical equity and agency resources. As noted,} how an audit proceeds depends upon the type of noncompliance suspected: for example, many audits on lower-to-middle income individuals concern a potentially incorrectly claimed tax \emph{credit}, whereas audits on higher income individuals 
more often involve insufficient taxes being paid on income or other assets~\cite{budget_dropoff}. Audits concerning tax credits are largely done via correspondence, where the IRS sends a letter to the taxpayer requesting verification of qualification for the claimed credit~\cite{budget_dropoff}. Other types of misreporting often incur in-person IRS audits
~\cite{budget_dropoff}. Correspondence audits are extremely resource-efficient for the IRS. 
On the other hand, in-person audits require more time and expertise, and tend to incur much higher costs. 
Further, a non-response from a correspondence audit is taken as an admission of non-compliance, resulting in revenue returned to the IRS~\cite{guyton2018effects}, and keeping investigation costs low. One study on EITC correspondence audits found that 
up to 75\% were determined to be noncompliant due to nonresponse, undeliverable mail, or insufficient response~\cite{guyton2018effects}. Thus, the ease of correspondence audits, coupled with the high nonresponse rate leading to frequent revenue returned to the IRS, may result in more reliably recovered income than in-person audits, in addition to their lower direct  costs. 
Here, we use a simple model to explore whether a constrained monetary budget, coupled with differential cost of audits across the income spectrum, might affect audit allocation. 
We model the audit budget in terms of a \emph{dollar} cost\footnote{We note that a fixed monetary budget may not perfectly capture the resource constraints faced in practice; for instance, the limited number of auditors of a given expertise level may bind more tightly than any short-term dollar budgets. Still, this simplification captures important heterogeneity in the degree to which audits push against agency resource constraints. In addition to shedding light on the status quo audit distribution, such analysis may be interesting to the field of applied ML, as relatively few papers consider budget-constrained allocation models.} as opposed to a constraint on the
fraction of the population 
audited. 

\vspace{-3mm}
\subsection{Experimental Details}
In our consideration of the effects of agency resource limitations on audit allocation, we focus on the dollar cost of audits to the IRS and its budgetary constraints. We calculate a simplified version of cost that only takes into account the cost of the actual tax examination, based on data from previous real operational audits. We calculate cost as the product of the examiner's time spent on a given audit with their hourly pay. 
We average this product over income deciles and \emph{activity code}, which roughly corresponds to groupings of individuals based upon what tax forms they have filled out, to estimate audit cost.
We incorporate cost into our analysis by directly including the dollar budget as an audit selection constraint, thus creating a linear program to maximize total predictive value (i.e. probability or amount of misreport) with respect to the dollar budget. As we show in Appendix~\ref{app:cost_eqs}, this formulation is equivalent to a fractional knapsack problem; thus, the optimal solution is to select individuals in order of their ratio of cost to return to the IRS, in other words, return-on-investment. 
We use a dollar budget of \$125M, the average estimated total cost of audits from years 2010-2014. 
Further details are 
in Appendix~\ref{app:cost_cac}.

\subsection{Results} 
We present three main results. First, due to the differing \emph{audit costs} to the IRS by income, {\hecom return-on-investment focused} audit
selection results in an allocation which overwhelmingly targets lower income taxpayers. In the left panel of Figure~\ref{fig:cost_constrained}, we show the optimal audit selection policy under a dollar budget with rankings from each of the models considered in our paper thus far. As described in Appendix~\ref{app:cost_eqs}, the revenue-optimal audit allocation is to choose returns with the return on investment, i.e. the best ratio of predicted reward (adjustment in regression or change probability in classification) to audit cost.   Based on our calculations of audit cost, audits in the highest income decile may cost up to 41 times the least costly audits. Given the disparities in audit costs over the income spectrum, the revenue-optimal audit selection method results in an allocation that almost exclusively targets lower income individuals. 

Second, the return on investment of auditing lower income individuals may shed light on the status quo allocation's focus on low and middle income individuals. We note that the optimal allocation with a dollar budget looks similar to the 2014 operational audit selection policy (Figure~\ref{fig:motivation}).
Given the decreasing IRS budget over time, prioritization of net revenue maximization may  influence the vertical equity of status-quo 
allocations. However, we note that the extremely low cost of audits on the lower end of the income spectrum result at least partially from a policy choice made to proceed with different types of audits in 
asymmetric ways: i.e., via \emph{correspondence audits} on the lower end of the spectrum, and in-person audits on the higher end. This decision, coupled with the choice to view a lack of response as noncompliance,  results in less time, and 
fewer resources, spent on audits for individuals in the lower end of the income spectrum, thus resulting in the constrained revenue-optimal allocation focusing so highly on low-income individuals.  

Third, we find that to improve vertical equity and increase revenue collected, regression models require a higher dollar budget.  
As demonstrated in Section~\ref{sec:changing} and Table~\ref{table:all_nums}, regression models produce the highest net revenue allocations amongst models 
constrained to only audit a given percentage of the population ($0.644\%$). 
However, 
the cost to the IRS of these allocations are considerably higher than classification methods---and indeed, higher than our approximation of average IRS budget between 2010-2014, \$125M.  
At this low dollar budget, regression models under-perform on revenue compared to classification models, demonstrated in the right panel of Figure~\ref{fig:cost_constrained}: this is because regression models target individuals in the higher income realm, where the audit cost is greater, thus preventing such allocations from targeting enough individuals to generate high revenue returns. 
This suggests that increasing the dollar budget available for audits may 
present an opportunity for not only more net revenue, but also in a more equitable allocation of audits. 


\vspace{-3mm}
\section{Discussion\label{sec:discussion}}

Through this unique collaboration with the Treasury Department and IRS, we have studied the impact of machine learning on vertical equity. Our work suggests that: (1) more accurate \emph{classifiers} may exacerbate rather than improve income fairness concerns
; (2) off-the-shelf fairness solutions are not well-suited for attaining income fairness; (3) fundamental modeling changes, like switching from a binary target to a regression target, can improve income fairness; and (4) external constraints, like institutional budgets, may influence fairness regardless of what underlying predictive model is used. {\hecom Specifically, a return-on-investment focused audit allocation may undermine vertical equity under current conditions}. More broadly, this work underscores the importance of vertical equity, in addition to horizontal equity, in real-world application areas of machine learning. 
To our knowledge, the term does not appear in the algorithmic fairness literature,\footnote{Outside the fairness community, but inside the general umbrella of technology and engineering, the term \emph{has} been used; in particular, \cite{yan2020fairness} use both terms in a study of equity in access to transportation, and point towards a possible link to algorithmic fairness. However, their interpretation of vertical and horizontal equity are substantially different from ours; for instance, they suggest that group fairness should be linked to \emph{vertical} equity.} and traditional fairness metrics can be seen as focusing on horizontal, rather than vertical, equity. 
Given the importance of achieving vertical equity for policy, this work points towards further development of algorithmic fairness techniques as a promising path for future research.

Our results also reveal a subtle dimension of fairness when resources are allocated under a budget constraint.  When there is greater uncertainty for high-income individuals, classification risk scores can shift audit allocations to lower-income individuals simply because misreports are easier to predict. Exploring the role of heterogeneity in uncertainty and its fairness implications might explain a wide range of other policies that have disparate impact (e.g., enforcement against blue collar vs. white collar crime).  In the tax context, this insight also underscores the need for information collection mechanisms (e.g., third party reporting by offshore financial institutions) to reduce such uncertainty in the high income space, which has been the subject of significant policy debate~\cite{,data_collection_taxes,edelhertz1970nature}.

We conclude by noting several limitations and opportunities for further work. First, we do not have access to the exact models employed by the IRS or the complete procedures, so we cannot make definitive inferences about past or current practice. Second, we only observe (an imperfect proxy of) the IRS cost of an audit, not taxpayer costs; the true societal cost of an audit may thus be materially different than what is used in Section \ref{sec:cost}. Third, our approach has not distinguished between underreporting from misreported income versus over-claimed refundable credits; some policymakers may view these forms of noncompliance differently. Finally, while the notion of monotonicity is motivated in part by the near-monotonicity of adjustments and the oracle results, it is not grounded in a full welfare analysis. Such an approach might take into account audit costs to taxpayers, deterrence effects, and other policy levers, such as tax rates or penalty amounts. 
Accounting for these dimensions may not necessarily yield strict monotonicity as a form of vertical equity, and we view this theoretical development as an important path to refining vertical fairness. 

Despite these limitations, this work represents an important step given the policy significance and complexity of this setting. The scale of the problem is substantial --- amongst U.S. taxpayers alone, improvements in this area can affect more than 100M individuals annually. Moreover, ``government by algorithm'' continues to grow \cite{engstrom2020government}, and understanding how to incorporate fundamental fairness and redistribution concerns in taxation may serve as a model for other governance-related settings. Finally, insights derived in this setting --- such as the differing effects of costs when considered as a constraint rather than in the objective --- may carry over to other unrelated settings. Our finding that a narrow  return-on-investment approach may degrade rather than improve vertical equity may be critical in a range of policy contexts~\cite{parrillo2013against}. Thus, both the technical concepts and policy problem are important and vital avenues for future research. 

\bibliographystyle{ACM-Reference-Format}
\bibliography{bib}


\clearpage
\appendix
\section*{Appendix}

    
\section{Fairness constraints and Monotonicity \label{app:intuition}}
\label{app:mono_proofs}
In this section, we show that a selection process which achieves either equal true positive rates or equalized odds will, under certain (differing) conditions, satisfy monotonicity with respect to the ranking of bins by true misreport rate. That is, such models must choose a higher audit rate in a group with a higher rate of misreport than it chooses in a group with a lower rate of misreport. Given that, in our setting, misreport rate appears to be monotonic with respect to income, such results would imply audit rate monotonicity with respect to income as well. 


For this section, we assume the following setup. There are two groups of observations $G_1$ and $G_2$ of equal size $n$, and they have $m_1$ and $m_2$ positive labels respectively and $r_1=n-m_1$ and $r_2=n-m_2$ negative labels. An auditor selects $A_1$ observations for audit from $G_1$ and $A_2$ from $G_2$ such that the total audits $A_1+A_2$ is their audit budget $A$. The auditor has access to a model $\mathcal{M}$ which gives binary predictions $\hat{y} \in \{0,1\}$. The auditor would like to select $A_1$ and $A_2$ in such a way that she maximizes true positives selected; we assume that $A<< \sum_{j\in\{1,2\}} \sum_{i \in G_j} \mathcal{M}(X_i)$ - that is, the audit budget is much smaller than the total amount of positive predictions by the model.

\emph{After} the auditor makes selections $A_1$ and $A_2$, we define the $\alpha_1$ as the false positive rate of the audits for $G_1$; that is, 
\begin{align*}
    \alpha_1 = \text{FPR}_1 = \frac{\text{False Positives in $G_1$ selected}}{r_1}. 
\end{align*}
In other words, $\alpha_1$ is the false positive rate of the \emph{composition} of whatever the auditor's selection process is with the predictions of the model (not the false positive rate of the model itself). We define $\alpha_2$ similarly. Additionally, we define $\beta_1$ as the true positive rate of the audits for $G_1$, i.e.:
\begin{align*}
    \beta_1 = \text{TPR}_1 = \frac{\text{True Positives in $G_1$ selected}}{m_1}
\end{align*} and $\beta_2$ similarly. Finally, let $p_i = \frac{\text{True Positive Predictions for group $i$}}{A_i}$, often known as precision.

\subsection{Equal TPR and Monotonicity}
Our first lemma relates monotonicity to precision in the case of a selection process satisfying equal true positive rates:
\begin{lemma}
Suppose that the selection process satisfies equal true positive rates. Then with $A_i,\  m_i,$ and $p_i$ defined as above: $A_2 \geq A_1 \iff \frac{m_1}{m_2} \leq \frac{p_1}{p_2}$.
\end{lemma}
\begin{proof}
Note that:
\begin{align*}
    p_i = \frac{\text{True Positive Predictions}}{\text{All Positive Predictions}}\implies \text{True Positives}_i = A_i p_i.
\end{align*}
Then the true positive rate can be written as
\begin{align*}
    \beta_i = \frac{\text{True Positive}_i}{\text{Positives}_i} = \frac{A_i p_i}{m_i}.
\end{align*}
But by assumption, $\beta_1 = \beta_2 = \beta$, so
\begin{align*}
    \frac{A_1 p_1}{m_1} = \frac{A_2 p_2}{m_2}.
\end{align*}
But this implies that \begin{align*}
    \frac{A_1}{A_2} = \frac{m_1}{m_2} \frac{p_2}{p_1}.
\end{align*}
Hence, $A_2 \geq A_1$ if and only if $\frac{m_1}{m_2}\frac{p_2}{p_1} \leq 1$, or in other words:
\begin{align*}
A_2 \geq A_1 \iff \frac{m_1}{m_2}\leq    \frac{p_1}{p_2}.
\end{align*}
\end{proof}
To interpret this lemma, suppose that Group 2 has a higher misreport rate than Group 1 by some factor. Then the lemma states that for any selection process satisfying equal true positive rates, monotonicity with respect to misreport rate requires precision in Group 2 greater than in Group 1 by at least the same factor, and vice versa. 
\subsection{Equalized Odds and Monotonicity}
%

The following lemma shows that, in this setting, any allocation that satisfies equalized odds (i.e. $\alpha_1=\alpha_2=\alpha$ and $\beta_1=\beta_2=\beta$) must audit the group with a \emph{higher} misreport rate at a \emph{higher} rate if the true positive rate is \emph{larger} than the false positive rate; conversely, it must audit the group with a \emph{higher} misreport rate at a \emph{lower} rate if the true positive rate is \emph{lower} than the false positive rate. 
\begin{lemma}\label{lemma:eo_mono}
Suppose that the allocation $A_1, A_2$ satisfies equalized odds. That is, $\alpha_1=\alpha_2=\alpha$ and $\beta_1=\beta_2=\beta$. If $\beta \geq \alpha$, then $A_2 \geq A_1 \iff m_2 \geq m_1$; otherwise, $A_2 \geq A_1 \iff m_1 \geq m_2$.
\end{lemma}
\begin{proof}
Note that $A_1$ is the sum of true and false positives in $G_1$ and $A_2$ is the sum of true and false positives in $G_2$. Since
\begin{align*}
    \alpha = \alpha_1 = \frac{\text{FP}_1}{r_1} \qquad \text{and} \qquad \beta = \beta_1 = \frac{\text{TP}_1}{m_1},
\end{align*}
we can observe that:
\begin{align*}
    A_1 = r_1 \alpha + m_1 \beta
\end{align*} 
and similarly for $A_2$. But then:
\begin{align*}
    A_2 - A_1 &= r_2 \alpha +m_1\beta - \left(r_1 \alpha + m_1 \beta\right)
    \\&=\alpha (r_2-r_1) + \beta(m_2-m_1)\\&= \alpha ((n-m_2)-(n-m_1)) + \beta (m_2-m_1) \\&= \alpha (m_1-m_2) + \beta (m_2-m_1) \\
    &= (\beta-\alpha)(m_2-m_1).
\end{align*}
But then we have that:
\begin{align*}
    A_2 - A_1 >0 \iff (\beta-\alpha)(m_2-m_1) >0,
\end{align*}
yielding the claimed result.
\end{proof}
Lemma \ref{lemma:eo_mono} shows that if the selection process as a whole satisfies equalized odds, then groups with higher misreport rates will be audited at a higher rate if and only if the process catches a larger fraction of misreporters than the fraction of non-misreporters it ensnares. In balanced settings and with good models, we might expect that generally the true positive rate will be higher than the false positive rate, and this is what provides intuition that imposing equalized odds might push the process towards monotonicity in misreport rate. But these rates interact with the overall audit budget: in the regime where the budget is very small and models are good, then it may be possible to obtain a low false positive rate but an \emph{even lower} true positive rate. In that case, equalized odds will require that the group with higher non-compliance is audited \emph{less}.

\section{Further Experimental Details}
\label{app:exp_details}
\begin{figure}
  \centering
    \includegraphics[width=\textwidth]{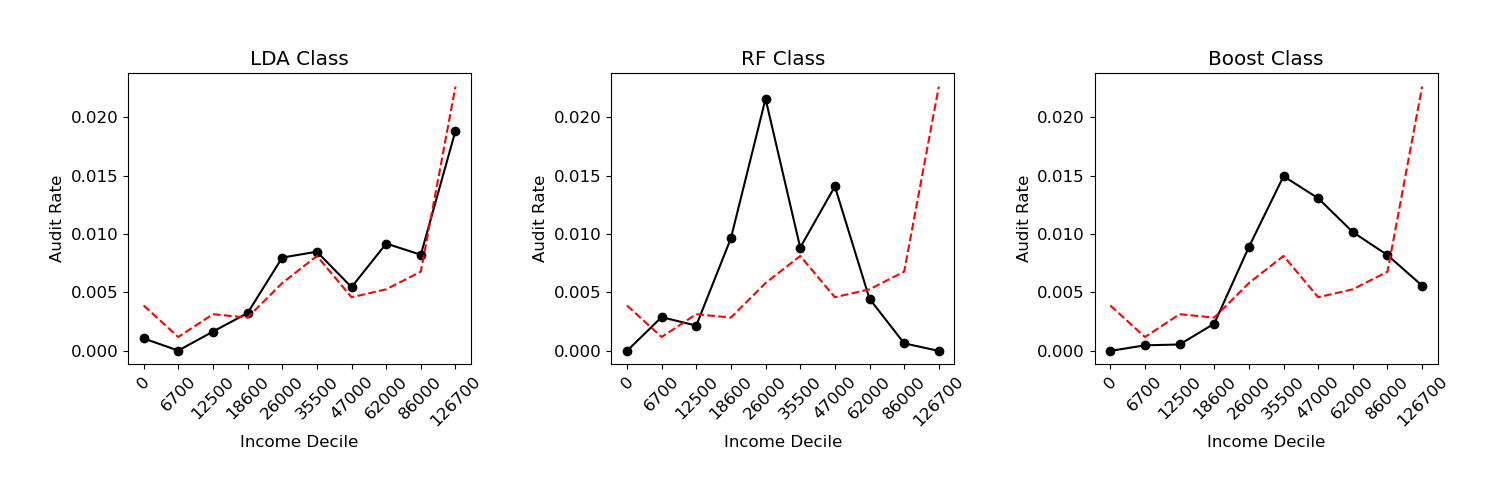}
    \caption{Audit rate over income deciles, for LDA, Random Forest, and XGBoost classifiers trained with unweighted datasets of size 100k, subsampled from the weighted NRP data.  (These allocations are in black, with oracle in red). 
    \label{app:fig:subsamp}}
\end{figure}

\begin{table}[t]
\centering
\small
\begin{tabular}{llc|ddddd}
\MC{1}{c}{Model Type} & \MC{1}{c}{Label} & \MC{1}{c|}{Subsampled} & \MC{1}{c}{Revenue}  & \MC{1}{c}{No-Change} & \MC{1}{c}{Cost} & \MC{1}{c}{Net Revenue} &\MC{1}{c}{Oracle}\\
 & \MC{1}{c}{Type} & \MC{1}{c|}{(Data Size)} &  \MC{1}{c}{(\$B)}  & \MC{1}{c}{Rate} &  \MC{1}{c}{(\$B)} & \MC{1}{c}{(\$B)} & \MC{1}{c}{Overlap}\\
\hline
Oracle & - & \texttimes & 29.40 & 0.0\%  &  0.33 & 29.07 & 1.00\\
LDA  & Class & \checkmark 11M & 6.07 & 12.8\% & 0.21 & 5.86 & 0.09\\
LDA  & Class & \checkmark 1100k & 6.61 & 16.0\% & 0.30 & 6.31 & 0.09\\
Random Forest & Class & \texttimes &  3.05 & 3.5\% & 0.08 & 2.97 & 0.00\\
Random Forest & Class & \checkmark  1100k &  3.19 & 4.5\% & 0.07 & 3.12 & 0.01\\
Grad Boost & Class & \texttimes & 4.05 & 4.2\% & 0.08 & 3.97 & 0.00\\ 
Grad Boost & Class & \checkmark  1100k & 3.72 & 4.7\% & 0.09 & 3.61 & 0.00\\
\end{tabular}
\caption{
Revenue, No-change rate, cost, and net revenue for models trained on a subsampled dataset of size 100k. No-change rate represents the percentage of audits that were allocated to compliant tax-payers; cost reflects cost to the IRS as described in Section~\ref{sec:cost}. These results reflect audit allocations which select the top $0.644$\% of taxpayers predicted most likely to misreport from each model. All metrics are reported on the test set, weighted using the sampling weights provided by the IRS to scale up to a representative sample of the US population. \label{app:tab:subsamp}}
\vspace{-0.2in}
\end{table}

In this paper, we compare LDA, Random Forest Classifier, Random Forest Regressor, Gradient Boost Classifier, and Gradient Boost Regressor models. We use the \emph{sklearn} python package \cite{scikit-learn} to implement all models except for gradient boosted models, and search for optimal hyperparameters using \textit{sklearn}'s \emph{GridSearchCV} method with 5-fold cross validation. Gradient boosted models are created through the XGBoost python package, and optimal hyperparameters are also found using GridSearchCV. 
We use NRP data from 2010-2014 to train all models in this paper, with dollar values scaled to 2014 values. Our threshold for determining what qualifies as a tax misreport is a \$200 difference between paid tax and amount owed.
We winsorize amount of misreport to the 1st and 99th percentiles. We split the data into train, test, and validation sets randomly. Our train and validation sets comprise 75\% of the data, with a test set of 25\% of the data.

We note that the IRS NRP data contains sampling weights, which are used to ensure that the NRP data is representative of the true underlying distribution of taxpayers~\cite{sample_weights}. 
We train all unconstrained models with sampling weights included in the NRP data using \textit{sklearn}'s built in data-weighting feature, except LDA, whose \emph{sklearn} implementation does not does not support training weights. 
For LDA, we create a representative dataset from the NRP data by randomly subsampling rows from the weighted training data according to the weights. For example, consider that each row $x$ has a weight $w$, and the sum of all weights in the training set is $W$. Then each observation has probability $\frac{w}{W}$ of getting selected as any given row in the subsampled data.
This produces an unweighted training set reflecting the same proportions as the weighted training data, with one million samples. As mentioned in Section~\ref{sec:fairlearn}, the \textit{FairLearn} package~\cite{bird2020fairlearn} requires the use of the \textit{sklearn} training weights feature to implement its in-process fairness enforcement algorithms. As a result, we also use the subsampling technique to create training sets for in-process fairness models, but with samples of $100,000$ points, as the algorithm is extremely time-intensive on large datasets (over 48 hours for one model). In order to show that the use of sampling weights during training, or the difference in training set size from 100k to 1M, does not strongly affect the results presented in the paper, we show the audit allocations and revenue, cost, and no-change rates of the LDA, Random Forest, and XGBoost classifiers in Figure~\ref{app:fig:subsamp} and Table~\ref{app:tab:subsamp} respectively.

All analyses sections are produced on the test set.  Cost and revenue calculations are reported by rescaling costs and revenues to reflect estimated annual values for the full population, for each year 2010-2014, and then dividing by five. 

We sort taxpayers by descending order of predicted \emph{misreport probability} 
from all classification models 
(using \textit{sklearn}'s \emph{predict\_proba()}) method,
in order to produce a ranking. We use \textit{sklearn}'s \emph{predict} method to return expected misreport for regression models.
We use an audit rate budget of $0.644$\% of the taxpayer population, reflecting the average audit rate from 2010-2014,  
and select audits $a_i$ by taking the top $0.644$\% of the taxpayer population in rank order. 
This $0.644$\% corresponds to weighted percentage of the population, computed with sampling weights, i.e. $\frac{\sum a_iw_i}{\sum w_i}$ where $i$ is an observation in the weighted dataset, $a_i$ is an indicator of whether to audit that observation, and $w_i$ is the number of people the observation represents to create a representative population from the sampling data. The audit budget of $0.644\%$ of the taxpayer population, is equivalent to 1125000 audits.

\section{Robustness Checks on Classification Thresholds}
\label{app:robustness_thresholds}

\begin{figure}
  \centering
    \includegraphics[width=\textwidth]{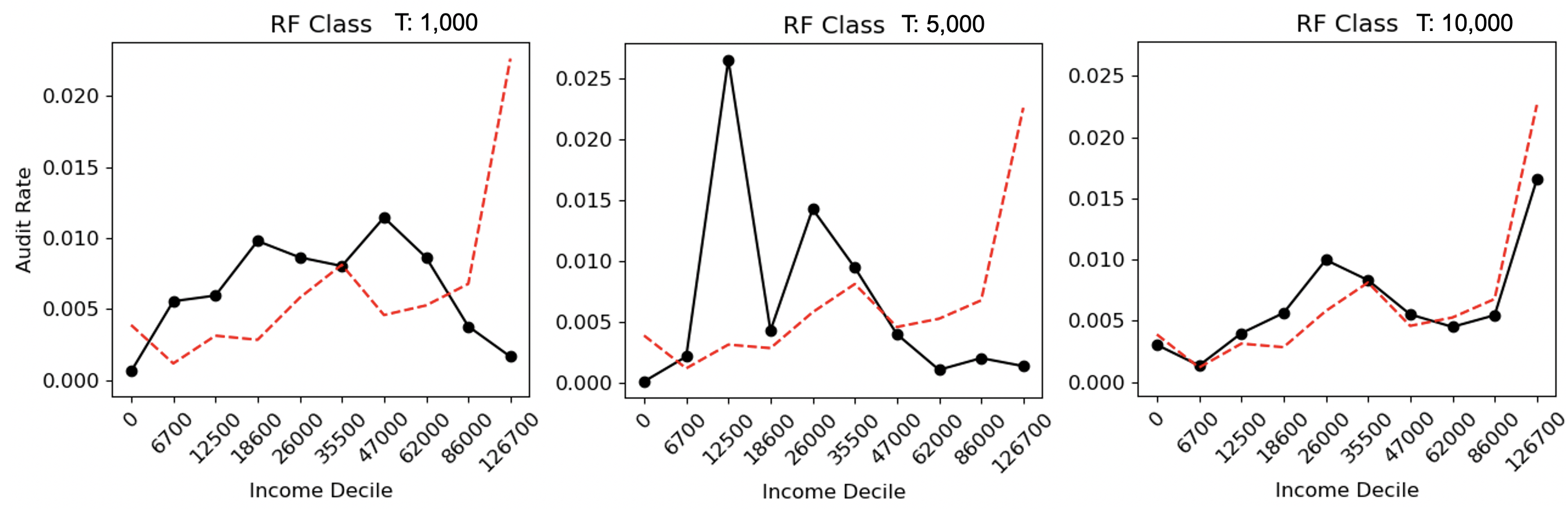}
    \caption{Audit rate over income deciles, for random forest classification models trained with different thresholds for what consitutes a significant amount of misreport. From left to right, we have the allocation for a model trained with a threshold of \$1,000, \$5,000, and \$10,000. (These allocations are in black, with oracle in red). 
    \label{app_fig:app_diff_thresh}}
\end{figure}

\begin{table}[t]
\centering
\small
\begin{tabular}{llc|dddd}
\MC{1}{c}{Model Type} & \MC{1}{c}{Label} & \MC{1}{c|}{} & \MC{1}{c}{Revenue}  & \MC{1}{c}{No-Change} & \MC{1}{c}{Cost} & \MC{1}{c}{Net Revenue} \\
 & \MC{1}{c}{Type} & \MC{1}{c|}{Threshold} &  \MC{1}{c}{(\$B)}  & \MC{1}{c}{Rate} &  \MC{1}{c}{(\$B)} & \MC{1}{c}{(\$B)} \\
\hline
Oracle & - & \texttimes & 29.40 & 0.0\%  &  0.33 & 29.07 \\
LDA  & Class & 200 & 6.07 & 12.8\% & 0.21 & 5.86\\
Random Forest & Class & 200 &  3.05 & 3.5\% & 0.08 & 2.97\\
Random Forest & Class & 1,000 & 4.92 & 5.6\% & 0.10  & 2.87 \\
Random Forest & Class & 5,000 & 6.48  & 43.6\% & 0.15 & 6.35\\
Random Forest & Class & 10,000 & 10.1  & 64.1\% & .45 & 10.55 \\
LDA & Class & 1,000 & 6.3 & 17.4\% & 0.20  & 6.1 \\
LDA & Class & 5,000 & 7.52  & 53.3\% & 0.30 & 7.22\\
LDA & Class & 10,000 & 9.0  & 70.8\% & .47 & 8.53 \\

\end{tabular}
\caption{
Revenue, No-change rate, cost, and net revenue for models with different thresholds for what constitutes a significant misreport. No-change rate represents the percentage of audits that were allocated to compliant tax-payers; cost reflects cost to the IRS as described in Section~\ref{sec:cost}. These results reflect audit allocations which select the top $0.644$\% of taxpayers (i.e. top 1125000 taxpayers) predicted most likely to misreport from each model. All metrics are reported on the test set, weighted using the sampling weights provided by the IRS to scale up to a representative sample of the US population. \label{app:diff_thresh_table}}
\vspace{-0.2in}
\end{table}

In this section, we compare the audit allocations of high-flexibility classification models (namely, random forest classifiers) with different thresholds for what constitutes a significant adjustment. In the main text, we use a threshold of \$200 to signify a significant misreport. In these experiments, we consider thresholds of 
\$1,000, \$5,000, and \$10,000. Experimental setup is identical to that described in Section~\ref{app:exp_details}, with the exception of the change in threshold. We display our results in Figure~\ref{app_fig:app_diff_thresh}, and Table~\ref{app:diff_thresh_table}. 

The results show us that changing the threshold of a significant adjustment to \$1,000 does not significantly impact audit allocation compared to the results presented in the main text. A threshold of \$5,000 exacerbates the classification model's excess focus on the lower end of the income spectrum, even beyond results shown in the main paper. Only a threshold of \$10,000 makes a significant difference in terms of the audit allocation---shifting the focus to high income individuals almost exclusively--- however, it results in an extremely high no-change rate.

\section{Increased Audit Focus on Lower-and-Middle Income only in High Complexity Models}
\label{app:log_reg}
\begin{figure}
  \centering
    \includegraphics[width=\textwidth]{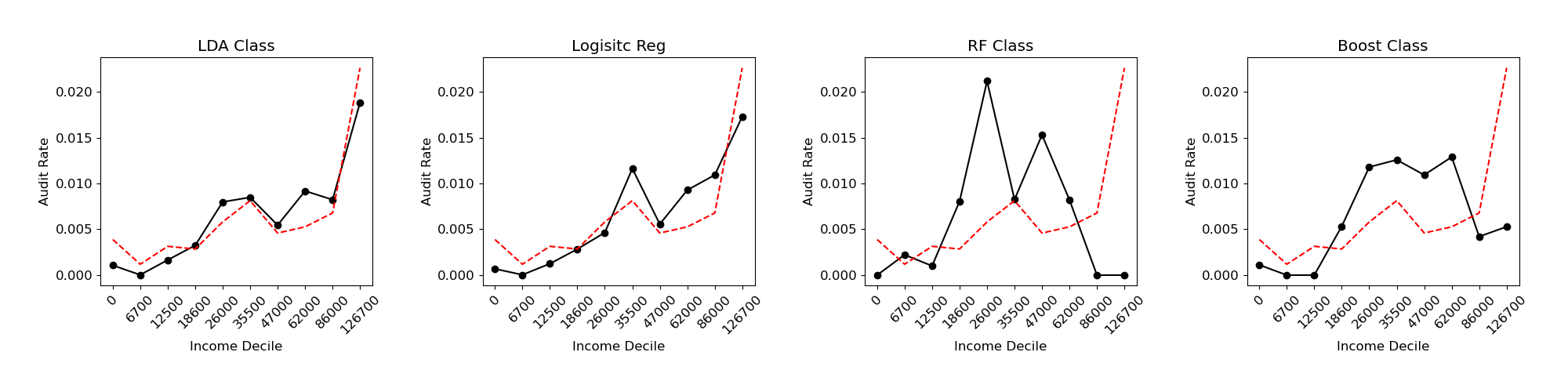}
    \caption{Audit rate over income deciles, for LDA, Logisitc Regression, Random Forest, and XGBoost classifiers trained on NRP data.  The new figure included in this graph, relative to the figures in the main paper, is the introduction of the logistic regression model. (These allocations are in black, with oracle in red). 
    \label{app:fig:log_reg}}
\end{figure}

\begin{table}[t]
\centering
\small
\begin{tabular}{llc|ddddd}
\MC{1}{c}{Model Type} & \MC{1}{c}{Label} & \MC{1}{c|}{Subsampled} & \MC{1}{c}{Revenue}  & \MC{1}{c}{No-Change} & \MC{1}{c}{Cost} & \MC{1}{c}{Net Revenue} &\MC{1}{c}{Oracle}\\
 & \MC{1}{c}{Type} & \MC{1}{c|}{(Data Size)} &  \MC{1}{c}{(\$B)}  & \MC{1}{c}{Rate} &  \MC{1}{c}{(\$B)} & \MC{1}{c}{(\$B)} & \MC{1}{c}{Overlap}\\
\hline
Oracle & - & \texttimes & 29.40 & 0.0\%  &  0.33 & 29.07 & 1.00\\
LDA  & Class & \checkmark  11M & 6.07 & 12.8\% & 0.21 & 5.86 & 0.09\\
Random Forest & Class & \texttimes &  3.05 & 3.5\% & 0.08 & 2.97 & 0.00\\
Grad Boost & Class & \texttimes & 4.05 & 4.2\% & 0.08 & 3.97 & 0.00\\ 
Log. Reg. & Class & \texttimes & 5.42 & 15.3\% & 0.19 & 5.23 & 0.06\\

\end{tabular}
\caption{
Revenue, No-change rate, cost, and net revenue for models presented in the paper alongside results for a logistic regression model. No-change rate represents the percentage of audits that were allocated to compliant tax-payers; cost reflects cost to the IRS as described in Section~\ref{sec:cost}. These results reflect audit allocations which select the top $0.644$\% of taxpayers predicted most likely to misreport from each model. All metrics are reported on the test set, weighted using the sampling weights provided by the IRS to scale up to a representative sample of the US population. \label{app:tab:log_reg}}
\vspace{-0.2in}
\end{table}

In this section, we provide results from a logistic regression model to further buttress the claim that only higher-complexity classification models result in audit allocations which exacerbate focus on lower and middle-income taxpayers. We train the Logistic Regression classification model with the same procedure outlined in Appendix~\ref{app:exp_details}, with sampling weights directly included during training. The audit allocation is depicted in Figure~\ref{app:fig:log_reg}: the allocation is more monotonic than the higher complexity classification models; and is apparent in Table~\ref{app:tab:log_reg}, the no-change rate is higher, but the revenue is higher as well.

\section{Additional Robustness Checks}
As noted in the main text, we make several important choices. First, we focus on total positive income (TPI), rather than adjusted gross income (AGI; roughly corresponding to the taxpayer's total net income) because it it represents a simple measure of earnings that is less likely to be affected by audit determinations. Second, for our analysis of the status quo, we do not differentiate between EITC-specific audits for EITC claimants (e.g. qualifying child eligibility) and income-centered audits (e.g. confirmation of reported small business or self-employment income). As we note above, this distinction is not relevant for the purposes of an ultimate determination as to a liability to the government, but for operational purposes, it may be meaningful to understand which type of audit is driving the vertical equity findings.
Third, we focus on reported income figures rather than audit-adjusted figures. This is because, by definition, audit-adjusted income is not available to the IRS before auditing, so any policy or choice that relies on access to audit-adjusted income is unimplementable. However, audit-adjusted income may provide a better picture of distributional effects (at least for audited taxpayers). 

\subsection{Status Quo}

In this section, we consider how the alternative choices (using AGI, splitting up EITC and income audits, and measuring model outcomes with respect to audit-adjusted income) in turn affect our status quo findings. We interpret these results as primarily confirming our main results. 

\paragraph{Adjusted Gross Income} 
First, we consider whether our motivating stylized facts --- that low-income taxpayers are audited at rates about as high as very-high income taxpayers despite change rate being monotonic in income and average adjustment being much higher for high income taxpayers --- is dependent on the choice of TPI rather than AGI. We thus recreate the left-most and right-most panels of Figure \ref{fig:motivation} with AGI as our feature in the x-axis. We use NRP data, which is selected via stratified random sampling, as before to avoid selection bias. 

The left panel of Figure \ref{fig:robust-agi} shows the 2014 audit rate for taxpayers in each \$10,000-wide bin of AGI. The figure shows that the large spike near $0$ observed with respect to TPI remains for AGI as well. However, the graph looks different in that AGI, unlike TPI, can be negative; the negative-AGI portion of the graph qualitatively resembles a (much noisier) mirror image of the non-negative-AGI porion, though negative-AGI taxpayers made up just over 1\% of all taxpayers according to NRP data. 

The right panel of Figure \ref{fig:robust-agi} depicts change rate and average adjustment across AGI bins. Here, the bins consist of AGI deciles for non-negative AGI taxpayers augmented by a single bin for all negative-AGI taxpayers. Excluding the negative-AGI bin, the change rate and average adjustment follow a qualitatively similar trend to their counterparts observed on TPI. That is, the change rate increases nearly monotonically, while the average adjustment is increasing overall but has a decreasing or flat portion. However, the overall difference between the average adjustment in the highest AGI bin and highest average adjustment among the lower-AGI bins is smaller than for TPI. As for the negative AGI bin, it has a relatively low (compared to other bins) change rate, but a higher average adjustment than any positive-AGI bin. Recall that AGI is income less various adjustments (e.g. for student loan interest, alimony payments, health insurance for self-employed taxpayers, etc.). As mentioned, given additional scope relative to TPI for errors, subjective determinations, or manipulation to influence ultimate AGI figures, we focus on TPI as our primary measure of income. 

\begin{figure}[H]
    \begin{subfigure}{0.45\textwidth}
    \centering
        \includegraphics[scale=0.30]{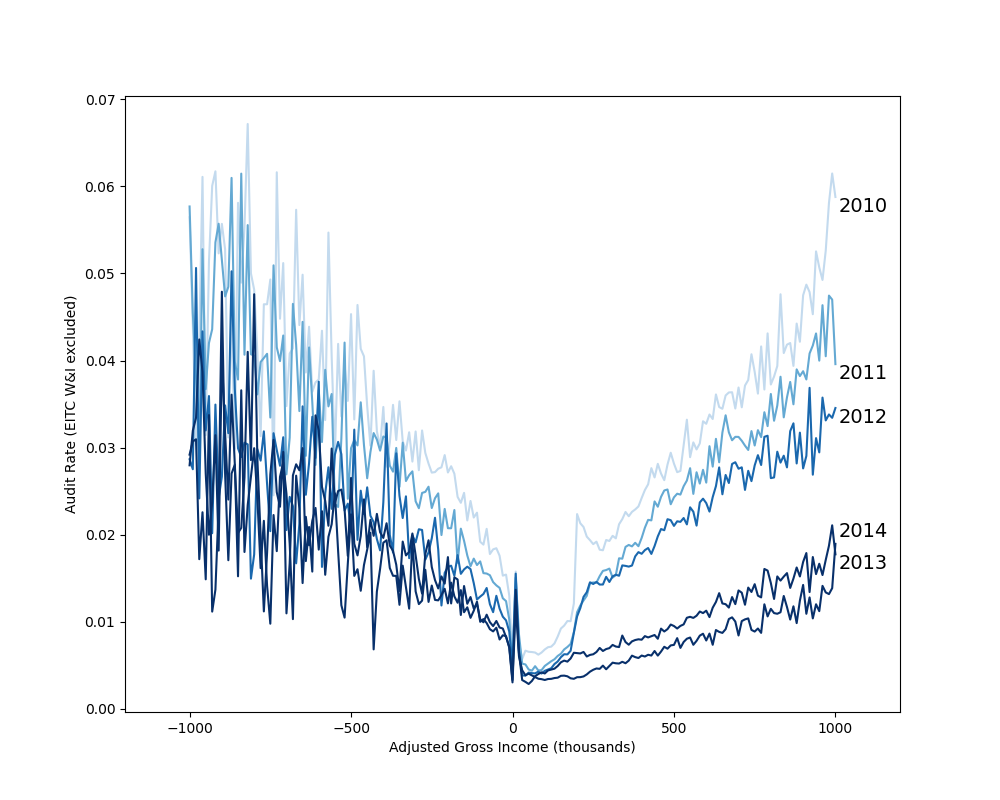}
\end{subfigure}
    \begin{subfigure}{0.45\textwidth}
    \centering
        \includegraphics[scale=0.30]{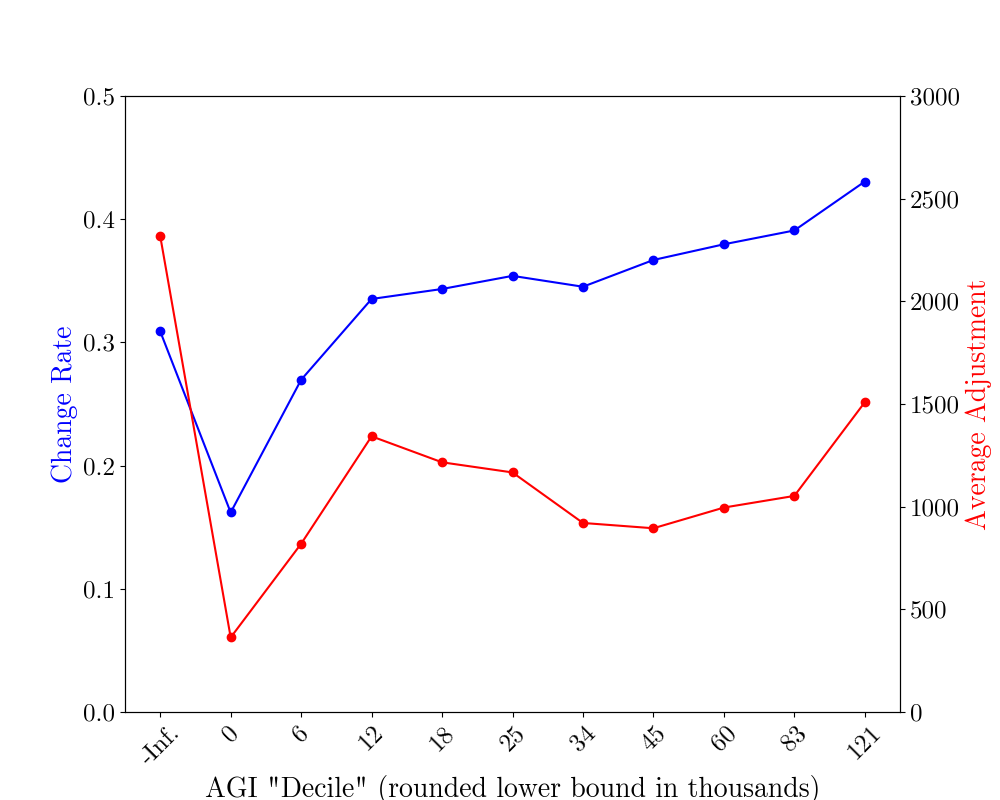}
\end{subfigure}
\caption{Robustness checks with adjusted gross income. Left: The figure shows the audit rate by year at a given amount of adjusted gross income (discretized into bins of \$10,000. Note that AGI may be negative; however, just over 1\% of NRP observations submit negative AGI, so the noise in the left half of the graph is due to small sample size. Right: The figure shows outcomes in terms of misreport rate and average adjustment by AGI ``deciles'' (we compute deciles for observations non-negative AGI and add all negative AGI observations as an additional initial bin). \label{fig:robust-agi}}
\end{figure}

\paragraph{Income vs. EITC Audits}
Next, we explore whether the extent to which the observed non-monotonicity in audit rates by income is driven primarily by income-related audits (e.g. verifying that claimed income was truly received, that reported income presents a full picture of true income, etc.) or eligibity-related audits (e.g., whether a claimed dependent satisfies residency or relationship tests for EITC eligibility). To do this, we replicate our main audit-rate analysis after removing dependent-related audits. 
We do this using \emph{project codes}. Projects codes are given to returns upon audit and correspond to a focus on particular issues. These do not necessarily map one-to-one with the income/EITC distinction --- for example, some project codes correspond to a particular flag being triggered, and can result in focus on both eligibility and/or income issues depending on the return; still, careful examination of the issues considered allow us to develop an approximate measure of the intent of the audit.\footnote{We started with a list of project codes, project titles, and project descriptions. We examined all projects with EITC-related words in the title (e.g.  ``EITC" or ``EIC"), as well as all projects indicated to be related to EITC by 4.19.14.4 in the Internal Revenue Manual.}  

We categorize EITC-related projects into three categories: most narrowly, \emph{EITC-eligibility projects}, which only consider questions related to whether a taxpayer's EITC claim satisfies eligibility requirements; more generally, \emph{EITC-Only} projects, which may consider more than eligibility but are still related to the EITC claim (e.g. verifiability of Schedule C income for EITC claimants); and most broadly, \emph{EITC-mentioning} projects, which constitute any project which mentions EITC as the population of interest. So, for instance, audits about the premium tax credit within EITC claimants would be considered as part of the \emph{EITC-mentioning} projects but not the \emph{EITC-Only} or \emph{EITC-eligibility} projects. Note that these categories are nested, so if we move from \emph{excluding} only the first to the next to the last we end up with a successively narrower set of included audits. In particular, the set of audits that fall into \emph{EITC-eligibility} projects but not \emph{EITC-Only} projects are those which correspond strictly to eligibility questions, and so the effect of removing them shows (a lower bound on) the portion of audits which are due to eligibility and not income. (It is a \emph{lower bound} because some projects in the \emph{EITC-Only} do not only focus on income, but may also focus on eligibility; without further detail unavailable in our data, we cannot further distinguish between specific issues considered for each return within the same project code.)

Figure \ref{fig:audits-no-eitc-projects} shows the results of this analysis for the tax year 2014. The figure depicts audit rate by TPI, but with several different lines indicating different levels of exclusions that have been made when calculating the audit rate. The shading increases with the breadth of exclusions (no exclusions, corresponding to our results in Figure \ref{fig:motivation}, are plotted in lightest red, while the broadest exclusions, of all projects with any mention of EITC at all, are plotted in darkest red). Notice that the lightest color shows the `spike' in audit rates for low income taxpayers, as displayed before, and excluding successively more returns unsurprisingly diminishes the calculated audit rate, until we are left with very few audits that are entirely unrelated to EITC claims for near-zero TPI taxpayers. Most interestingly, moving from no exclusions to excluding EITC-eligibility-specific projects decreases the audit rate at the spike from about 1.2\% to about .7\%. This indicates that, as a lower bound, about half of the spike is explained by EITC-eligibility-related projects. 


\begin{figure}[H]
    \centering
    \includegraphics[scale = 0.5]{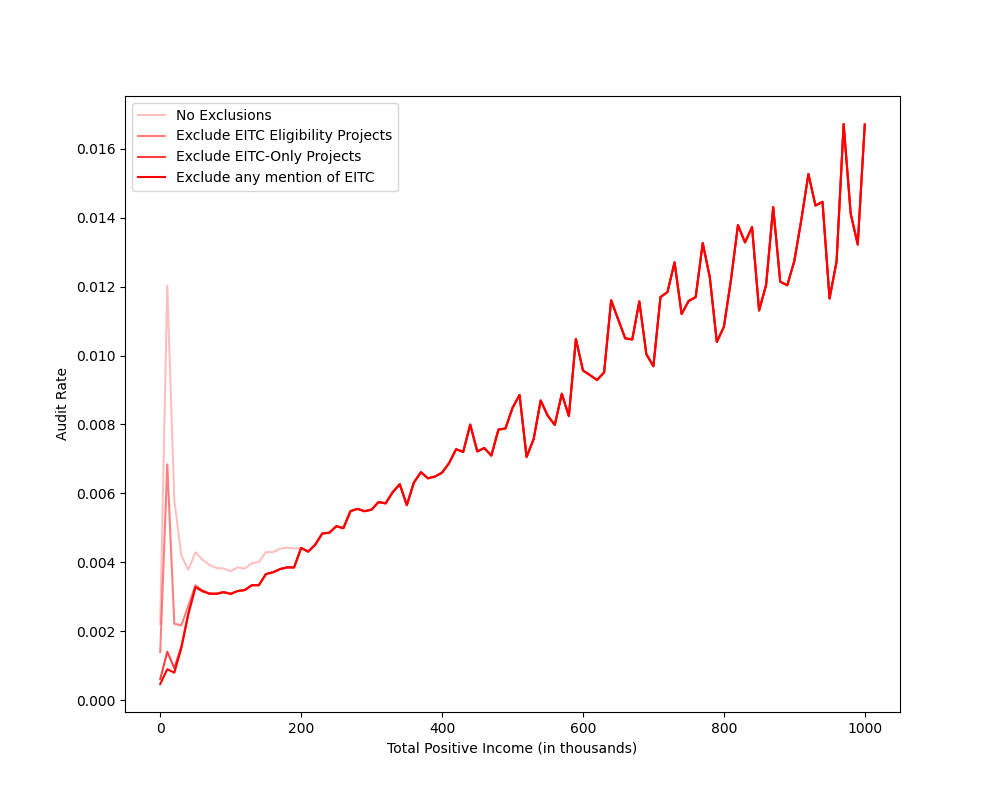}
    \caption{Audit rate by TPI for tax year 2014 after excluding EITC-related projects of varying stringency of definition. The shades of lines move from light to dark mirroring how the consider exclusions move from very little to very broad. In particular, the lightest shade shows audit rate before any exclusions, the next shows audit rate after excluding projects related specifically to EITC eligibility, the next after excluding all projects related \emph{only} to EITC, and the darkest after excluding all projects which mention EITC even if focused on unrelated issues. }
    \label{fig:audits-no-eitc-projects}
\end{figure}

More coarsely, we can simply look at to what extent the spike is being driven by EITC claimants at all, as indicated by claimants' \emph{activity codes}. Activity code 270 correspond to EITC claimants with less than \$25,000 of Schedule C (non-wage) income (e.g. income from self-employment), while activity code 271 captures the remainder. (Recall that income for the purposes of the EITC is not TPI, but AGI, as described above. So it is possible, though rare, for a taxpayer with high TPI to nonetheless be eligible for the EITC.) Figure \ref{fig:audits-no-eitc-acs} displays the results of a similar exercise, moving from excluding 270 to excluding 270 and 271. The fact that the spike is essentially eliminated moving from no exclusions to excluding 270 suggests that  non-monotonicity is driven by EITC claimants. (Note that this is not inconsistent with Figure \ref{fig:audits-no-eitc-projects} because EITC claimants in 270 may be audited for non-eligibility matters, like income verification.)   
\begin{figure}[H]
    \centering
    \includegraphics[scale = 0.5]{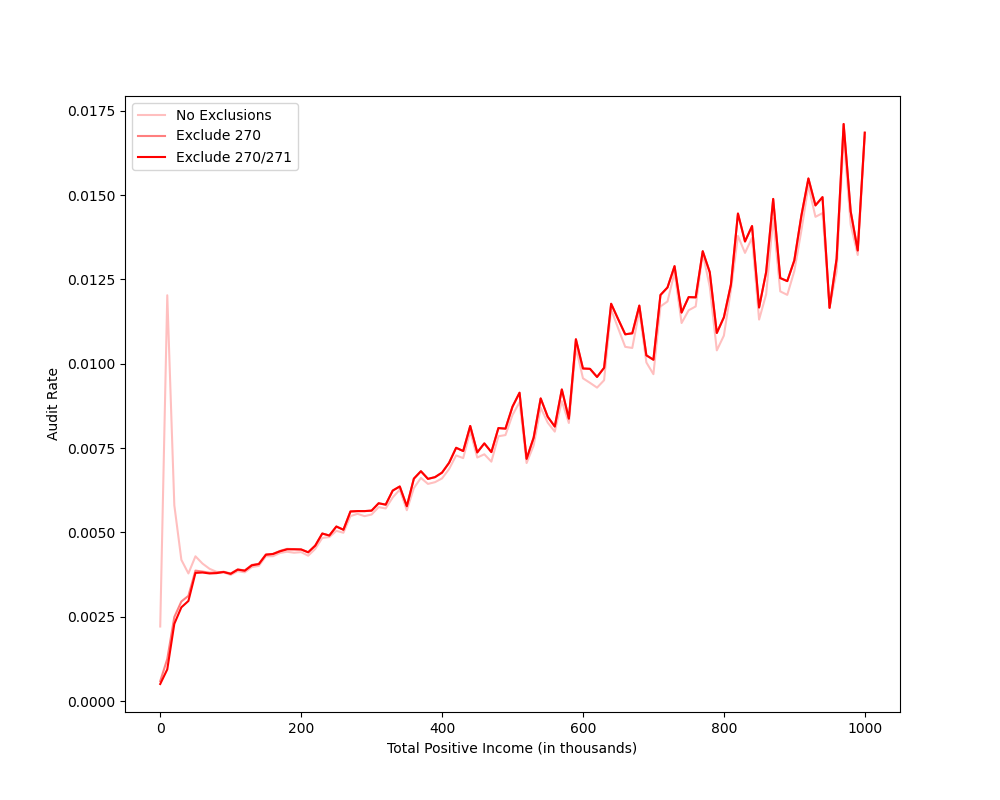}
    \caption{Audit rate by TPI for tax year 2014 after excluding EITC-related activity codes. The lightest line corresponds to the underlying audit rate without exclusions, the next darkest to the audit rate after excluding activity code 270, and the darkest to after removing 270 and 271 (i.e. all EITC claimants).}
    \label{fig:audits-no-eitc-acs}
\end{figure}

\paragraph{Outcomes with respect to true TPI} 
Finally, we recalculate no-change rates and average adjustments by corrected, rather than reported, TPI and AGI. (Note that since outcomes are measured in NRP, we have corrected incomes for nearly all taxpayers, modulo a small number of missing observations.) The outcomes are displayed in Figure \ref{fig:outcomes-cor}. Qualitatively, the TPI picture (left panel) looks similar to the right panel of Figure \ref{fig:motivation}, but with an even clearer monotonicity pattern in average adjustment, as the downward trend in adjustments in between the 3rd-7th bins of (uncorrected) TPI is replaced by a plateau. Moreover, measured according to corrected TPI, the average adjustment is higher in the highest-income bin than according to reported TPI, but lower in the lower-income bins; in other words, the overall trend is much starker for corrected than reported TPI. The AGI picture (right panel) appears qualitatively very similar to the TPI picture, indicating that monotonicity of change rate and adjustment holds regardless of income measure, at least after correcting for the truth.  
\begin{figure}[H]
    \begin{subfigure}{0.4\textwidth}
    \centering
    \includegraphics[scale=0.25]{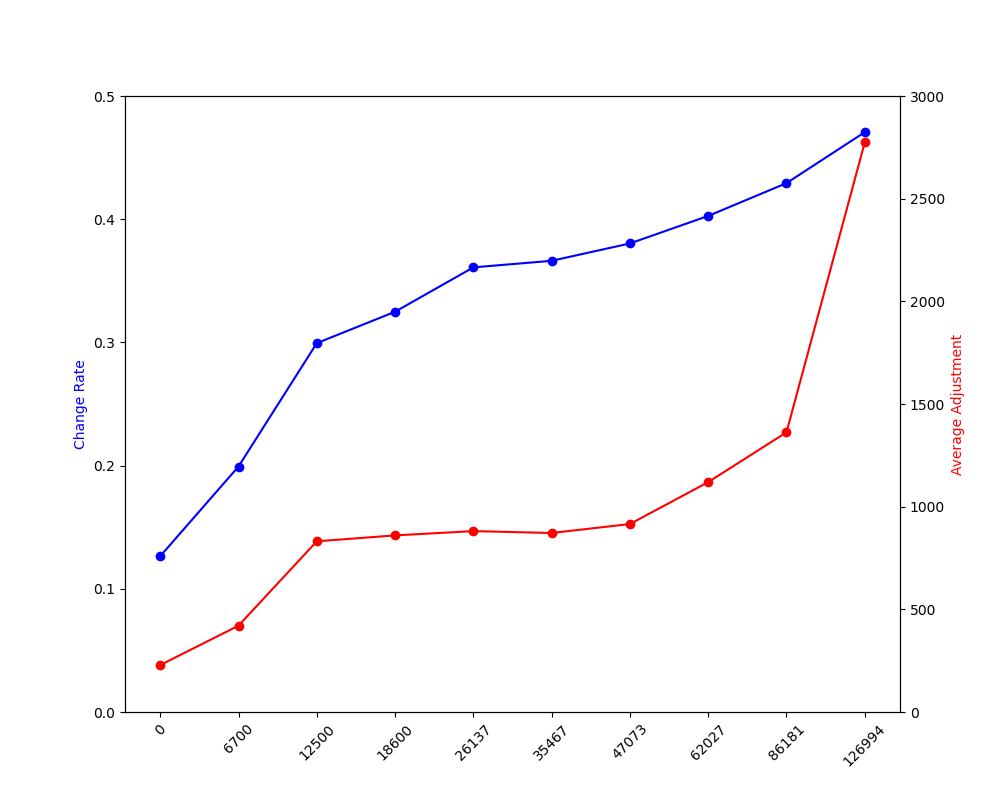}
\end{subfigure}
\begin{subfigure}{0.4\textwidth}
    \centering
        \includegraphics[scale=0.25]{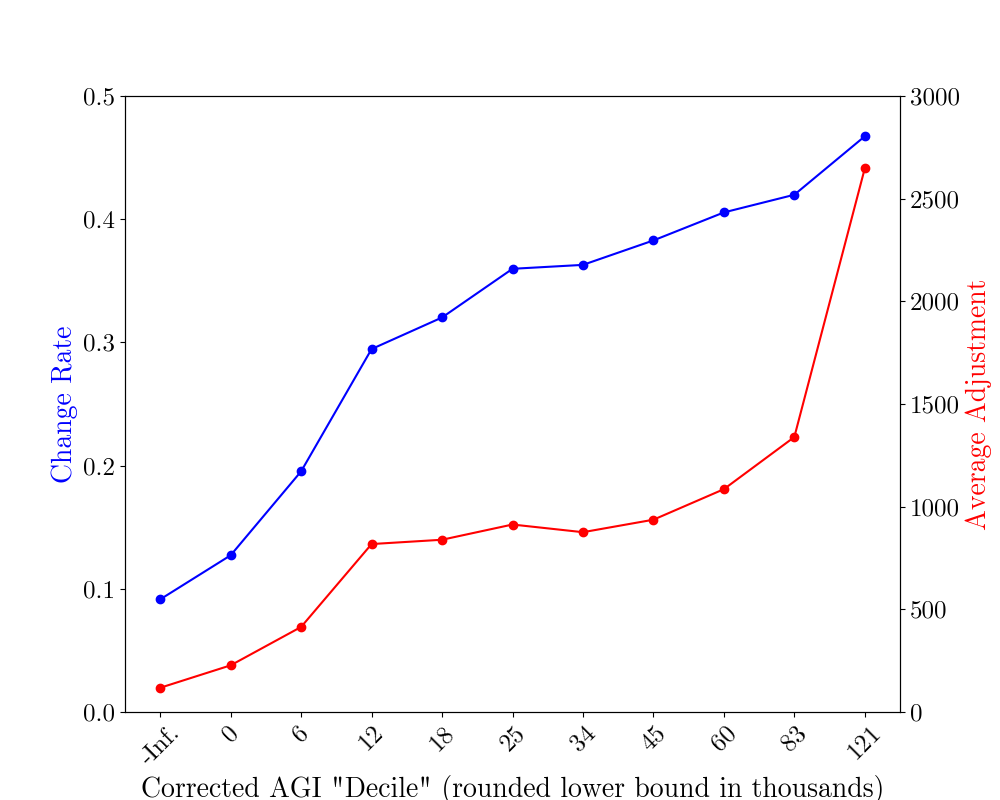}
\end{subfigure}
    \caption{The figures display outcomes --- no-change rate, in blue and measured on the left y-axis, and average adjustment, in red and measured on the right y-axis --- by corrected TPI (left panel) and corrected AGI (right panel). \label{fig:outcomes-cor}}
\end{figure}

\subsection{Fairness methods and Modeling Choices }

In this section, we display audit rate by income of classification, regression, and fairness-constrained models 
presented in the main paper, but with income buckets over \emph{audit-adjusted adjusted gross income} (AA-AGI), and \emph{audit-adjusted total positive income} (AA-TPI). This provides a robustness check to test whether models which display low audit focus on \emph{reported} low income  also do so on \emph{true} low income populations, and if this pattern carries over to other notions of income, such as taxable (and not total) income.


\paragraph{Experimental Setup.} For AA-TPI, we use the same income buckets as we have throughout the paper (which determine deciles on total positive income) for consistency and ease of comparison. 
For AA-AGI, we re-compute buckets, and also create a separate bucket for individuals with negative AGI, but note that they only make up approximately 0.7\% of the population (less than 1/10 of a decile), and thus the results on this population are not directly comparable to those on the rest of the deciles due to the vastly different sample size. For both measures of income, approximately 1,000 out of 71,000 rows do not contain audit-adjusted AGI or TPI, which we exclude from the analysis. 

\paragraph{Results}
The audit distributions over income deciles over AA-AGI and AA-TPI are  largely similar. For AA-AGI, the boosted regressor focuses slightly less on middle-to-high income.
For both AA-TPI and AA-AGI, the EO constrained classifier focuses lightly less on middle income individuals ($\sim$47k).
Regression and LDA models 
select a high rate for 
individuals with negative AA-AGI, but this is drawn from a very small percentage of the population (0.7\%).
Otherwise, the overall trends of audit focus for audit focus across the different classifiers remains the same. 

The most notable change from reported TPI to AA-TPI and AA-AGI is the extent to which the oracle focuses on ``truly'' high income individuals --- whereas the oracle audited up to 1\% individuals with zero and middling reported TPI, from the perspective of AA-TPI and AGI, the oracle focuses almost exclusively on the upper third of the income spectrum, and most dramatically (approx 4.5\%, as opposed to approx. 2\% for reported TPI) on the highest income decile.

\begin{figure}
  \centering
    \includegraphics[width=\textwidth]{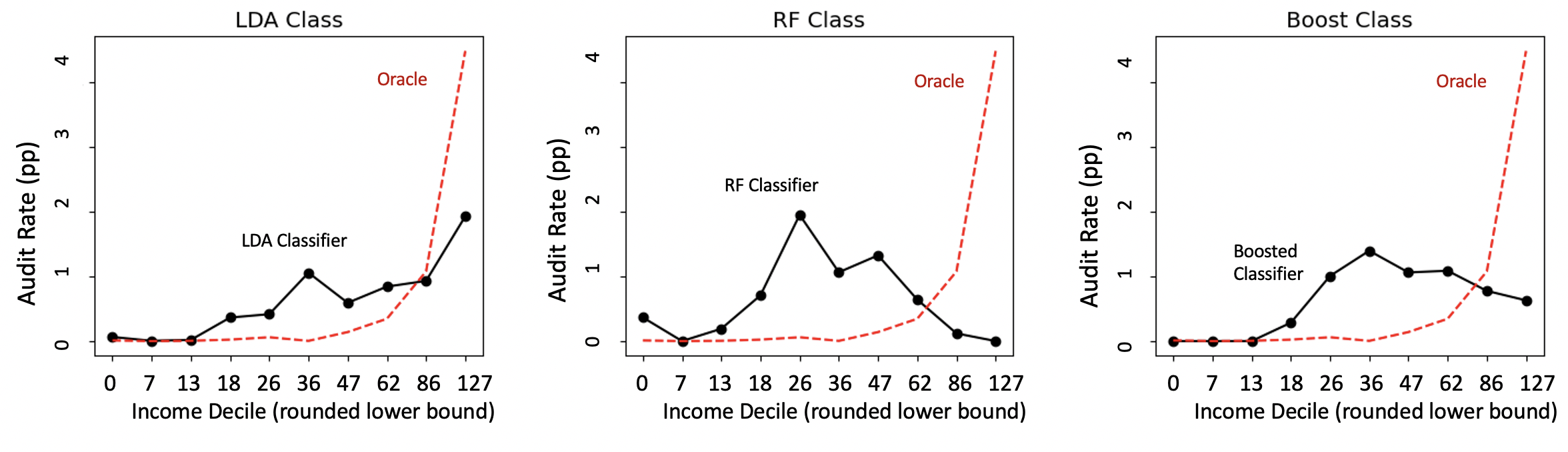}
    \caption{Audit rate by income for classification models.  From left to right: LDA classifier, Random Forest Classifier, and Boost Classifier. We use the same income deciles as presented throughout the paper for ease of comparison, but with corrected total positive income (after audit) as opposed to reported. Income decile lower bounds are given in thousands of dollars.
    \label{app_fig:class_models_tot_pos}}
\end{figure}

\begin{figure}
  \centering
    \includegraphics[width=4in]{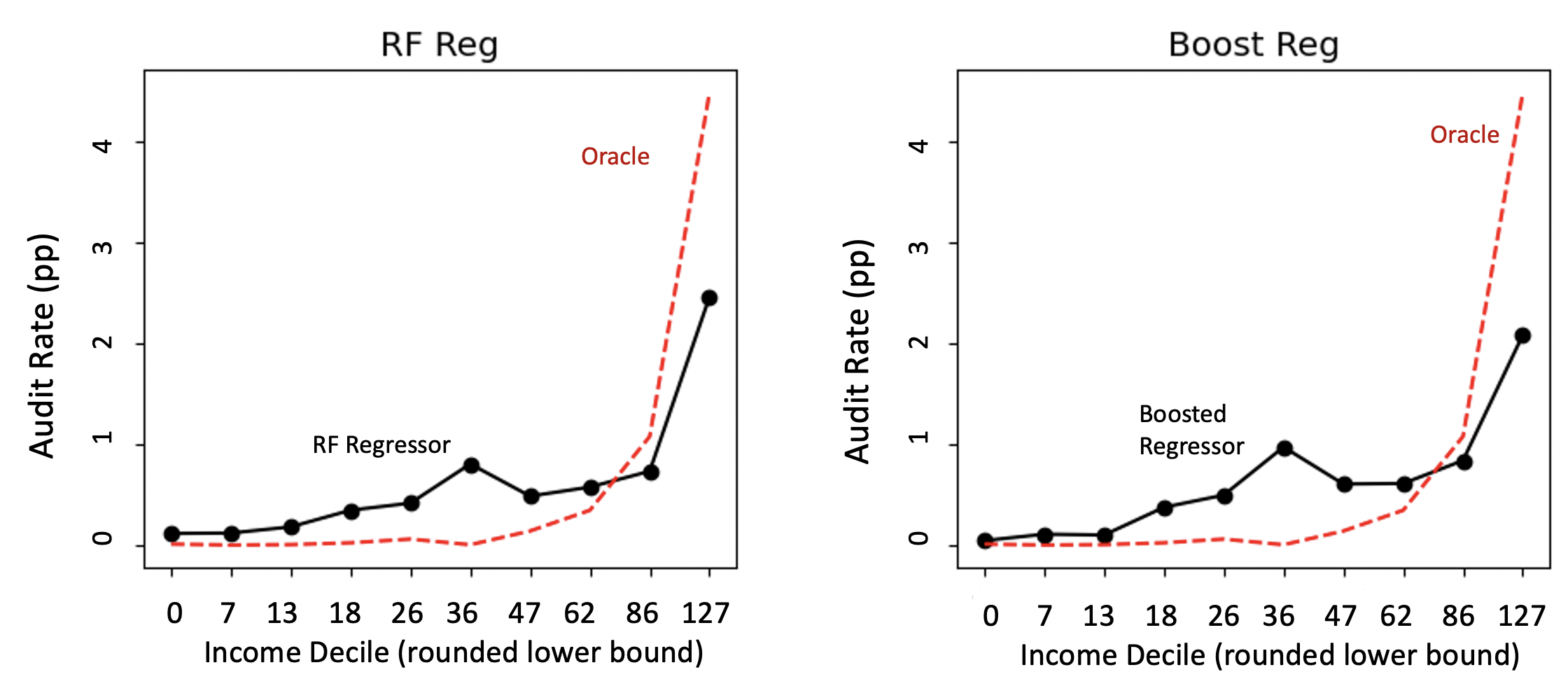}
    \caption{
    Audit rate by income for regression models. We use the same income deciles as presented throughout the paper for ease of comparison, but with corrected total positive income (after audit) as opposed to reported. Income decile lower bounds are given in thousands of dollars. 
    \label{app_fig:reg_models_tot_pos}}
\end{figure}

\begin{figure}
  \centering
    \includegraphics[width=\textwidth]{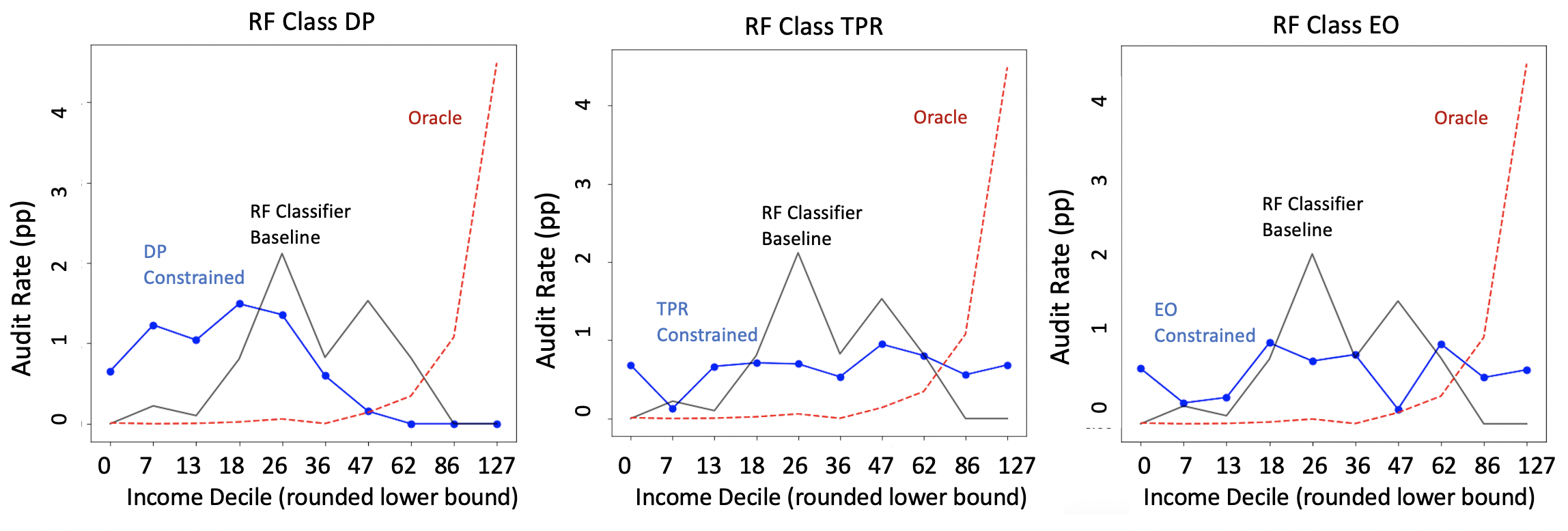}
    \caption{Audit rate by income from in-process fairness constrained random forest models, graphed over audited corrected TPI (AA-TPI).  We use the same income deciles as presented throughout the paper for ease of comparison, but with corrected total positive income (after audit) as opposed to reported. 
    \label{app_fig:fairness_tot_pos}}
\end{figure}

\begin{figure}
  \centering
    \includegraphics[width=\textwidth]{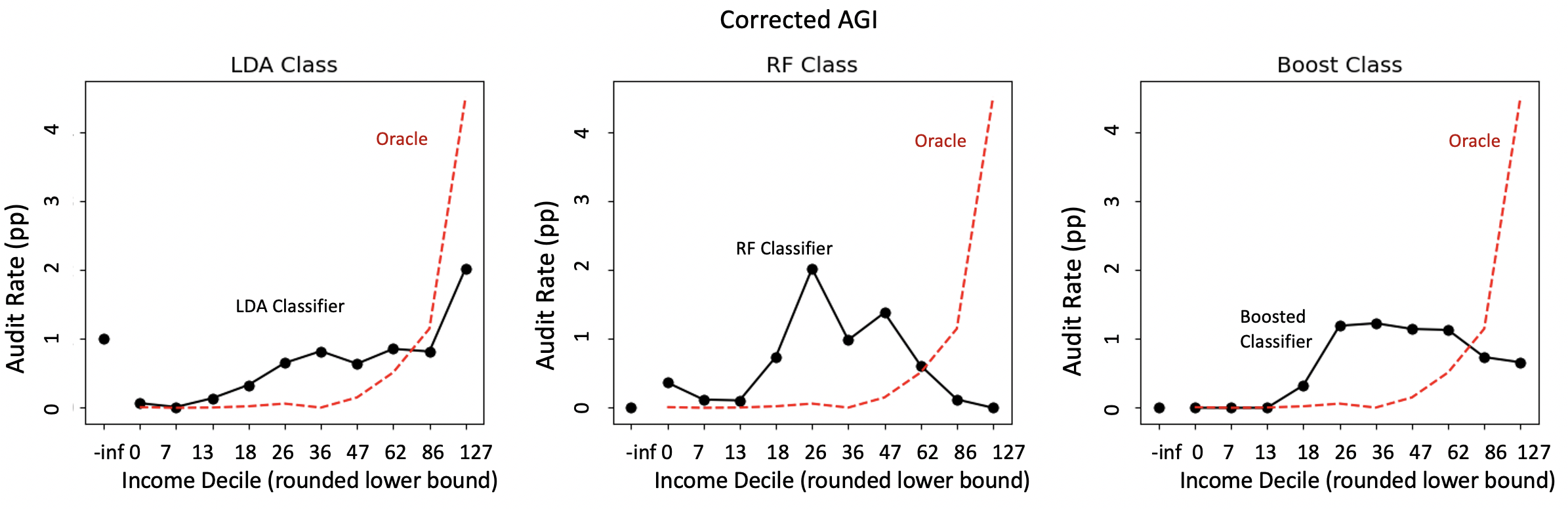}
    \caption{Audit rate by income for classification models. From left to right: LDA classifier, Random Forest Classifier, and Boosted Classifier. We plot over 10 AGI-derived deciles (0-127k are the lower-bounds), with an additional column for the taxpayers with negative corrected AGI. Note that the first column (-inf) is not a true decile, as individuals with true negative AGI make up less than 0.7\% of the population.
    \label{app_fig:class_models_AGI}}
\end{figure}

\begin{figure}
  \centering
    \includegraphics[width=\textwidth]{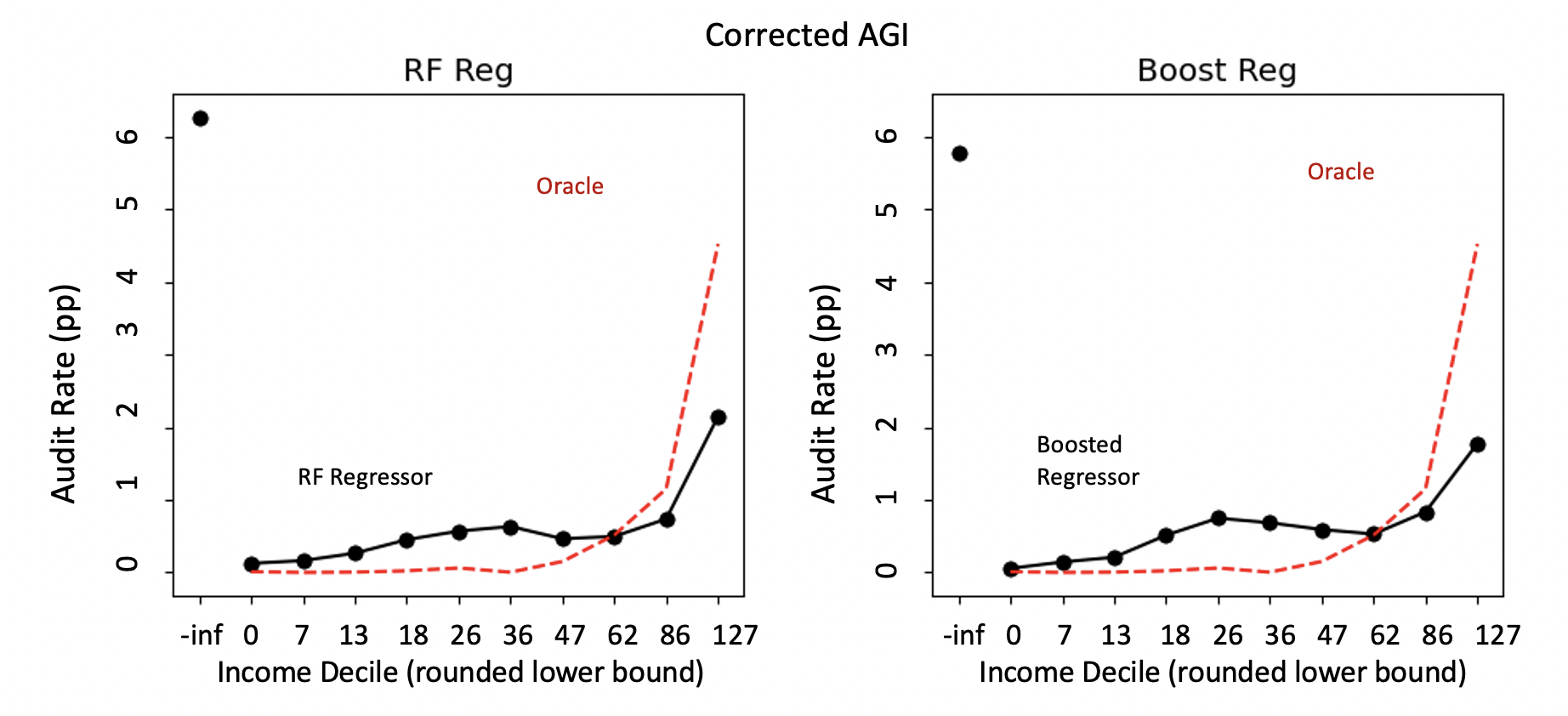}
    \caption{Audit rate by income for regression models. We plot over 10 AGI-derived deciles (0-127k are the lower-bounds), with an additional column for the taxpayers with negative corrected AGI. Note that the first column (-inf) is not a true decile, as individuals with true negative AGI make up less than 0.7\% of the population. 
    \label{app_fig:reg_models_AGI}}
\end{figure}

\begin{figure}
  \centering
    \includegraphics[width=\textwidth]{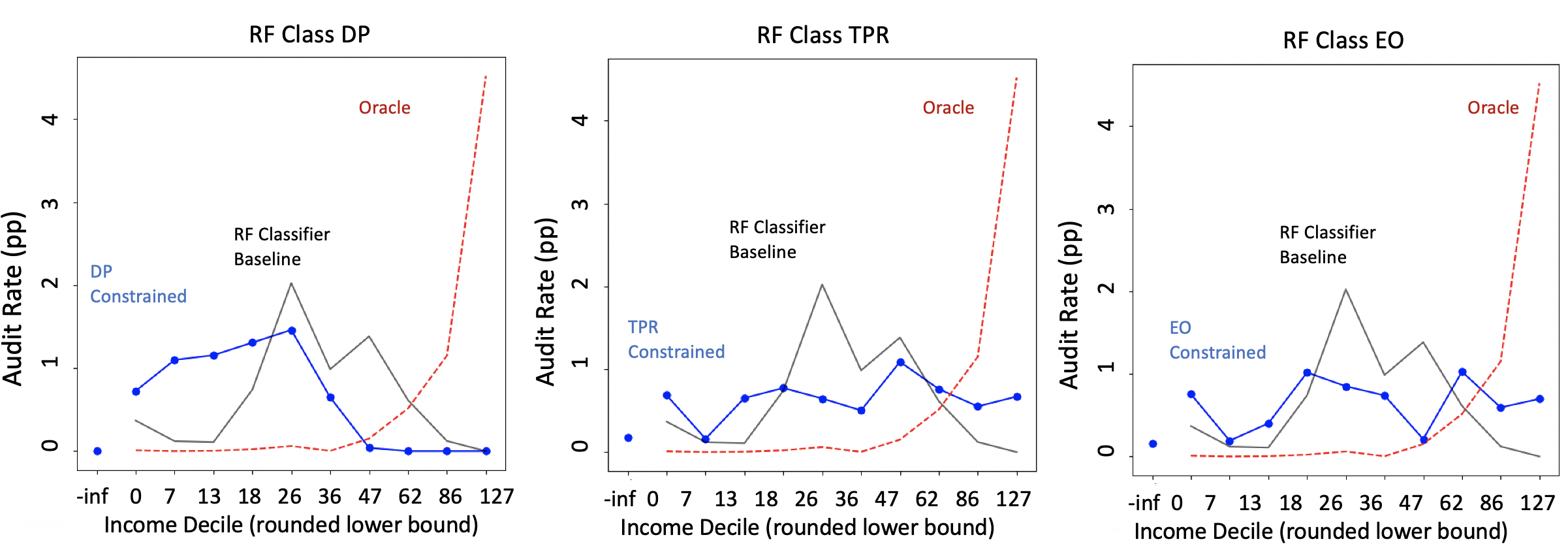}
    \caption{Audit rate by income from in-process fairness constrained random forest models, graphed over audited corrected AGI. We plot over 10 AGI-derived deciles (0-127K are the lower-bounds), with an additional column for the taxpayers with negative corrected AGI. Note that the first column (-inf) is not a true decile, as individuals with true negative AGI make up less than 0.7\% of the population. Income decile bounds are given in thousands.
    \label{app_fig:fairness_agi}}
\end{figure}

\section{Further Fairness Results}
\label{app:more_fairness}
In this section, we present complete in-processing results, and also show results from another technique, specifically, post-processing techniques for enforcing fairness constraints. We also discuss why pre-processing techniques, and perhaps counterintuitively, fair ranking methods are not well-suited to our setting. 

\subsection{In-processing}
As noted in Section~\ref{sec:fairlearn}, the in-processing results do not result in audit allocations which respect the fairness constraints the models are trained to obey, partially due to the fact that the audit allocation focuses only on the top 0.644\% of predictions. First, we present (i) numerical evidence that in-process fairness constrained models do not produce allocations which respect the constraints they are trained to satisfy (Tables~\ref{app:tab:fpr_budgeted} and ~\ref{app:tab:tpr_budgeted}), (ii) we show evidence that the in-processing results did perform according to expectation, i.e., they do produce models which satisfy their respective constraints over the \emph{full suite of predictions} on the training set, in Table~\ref{app:tab:sanity_check}. 

We present only numeric clarification for the fact that the allocations do not satisfy the constraints which are enforced on the model for true positive and false positive rates, as the fact that selection rate parity is not upheld is clear from the graph of the allocation (as an allocation which satisfies selection rate parity would have equal audit rate across all income groups). 

We note that we present the true and false positive rates calculated over the \emph{weighted} population---i.e. calculating all metrics taking into account the sample weight of each row---as well as over the unweighted raw data. This is due to the fact that the algorithm used to implement these results do not offer any guarantees over weighted data~\cite{agarwal2018reductions}. However, we find that the results are qualitatively similar.

\begin{table}[]
\begin{tabular}{l|llll|llll}
\toprule
& \multicolumn{8}{c}{In-Process Fairness Method: False Positive Rates}\\

& \multicolumn{4}{c}{Unweighted} & \multicolumn{4}{c}{Weighted (W)}\\
\midrule
Income \\ Bucket  & Uconstr. & SR PAR & TPR PAR & EO    & Unconstr. W & SR PAR W & TRP Par W & EO W  \\
\toprule
0                     & 0.000         & 0.008   & 0.002    & 0.011 & 0.000           & 0.006     & 0.000      & 0.005 \\
7                     & 0.000         & 0.008   & 0.003    & 0.001 & 0.000           & 0.009     & 0.002      & 0.004 \\
13                    & 0.000         & 0.012   & 0.002    & 0.006 & 0.000           & 0.010     & 0.001      & 0.003 \\
18                    & 0.000         & 0.016   & 0.000    & 0.002 & 0.000           & 0.010     & 0.000      & 0.007 \\
26                    & 0.006         & 0.009   & 0.002    & 0.006 & 0.004           & 0.007     & 0.000      & 0.006 \\
36                    & 0.000         & 0.003   & 0.005    & 0.015 & 0.000           & 0.002     & 0.001      & 0.003 \\
47                    & 0.000         & 0.000   & 0.000    & 0.007 & 0.000           & 0.000     & 0.000      & 0.009 \\
62                    & 0.000         & 0.000   & 0.004    & 0.010 & 0.000           & 0.000     & 0.003      & 0.012 \\
86                    & 0.000         & 0.000   & 0.005    & 0.008 & 0.000           & 0.000     & 0.004      & 0.018 \\
126                   & 0.000         & 0.000   & 0.000    & 0.008 & 0.000           & 0.000     & 0.000      & 0.007 \\
\toprule
& \multicolumn{8}{c}{Post-Process Fairness Method: False Positive Rates}\\

\midrule
Income \\ Bucket  & Unconstr. & SR PAR & TPR PAR & EO    & Unconstr. W & SR PAR W & TRP Par W & EO W  \\
\toprule
0                     & 0.000         & 0.000   & 0.000    & 0.000 & 0.000           & 0.000     & 0.000      & 0.000 \\
7                     & 0.000         & 0.000   & 0.000    & 0.000 & 0.000           & 0.000     & 0.000      & 0.000 \\
13                    & 0.000         & 0.000   & 0.000    & 0.000 & 0.000           & 0.000     & 0.000      & 0.000 \\
18                    & 0.000         & 0.000   & 0.000    & 0.000 & 0.000           & 0.000     & 0.000      & 0.000 \\
26                    & 0.006         & 0.006   & 0.006    & 0.000 & 0.004           & 0.004     & 0.004      & 0.000 \\
36                    & 0.000         & 0.000   & 0.000    & 0.000 & 0.000           & 0.000     & 0.000      & 0.000 \\
47                    & 0.000         & 0.000   & 0.000    & 0.000 & 0.000           & 0.000     & 0.000      & 0.000 \\
62                    & 0.000         & 0.000   & 0.000    & 0.000 & 0.000           & 0.000     & 0.000      & 0.000 \\
86                    & 0.000         & 0.000   & 0.000    & 0.000 & 0.000           & 0.000     & 0.000      & 0.000 \\
126                   & 0.000         & 0.000   & 0.000    & 0.020 & 0.000           & 0.000     & 0.000      & 0.022\\
\bottomrule
\end{tabular}
\caption{We present the false positive rates by income bucket for the audit allocations generated from unconstrained and fairness-constrained random forest classifier models on the \emph{test} set, where an audit allocation corresponds to the highest ranked predictions from each model up to a budget of 0.644\% of the taxpayer population, or 1125000 audits. Unconstr. refers to an unconstrained model, SR PAR to selection rate parity, TPR PAR to true positive rate parity, and EO to equalized odds. We note that the algorithms implemented in \textit{Fairlearn}\cite{bird2020fairlearn} only guarantee satisfying fairness constraints in expectation on the training set, over the entire set of predictions (i.e. not simply the top 0.64\%). Also note that the only column where we would expect to see equalized false positive rates is the equalized odds (EO) column(s). The top table represents results from in-process fairness methods, and the lower table from post-process fairness enforcement methods. The numbers in the left side (left four columns) of the table corresponds to the calculation on the raw data, without sample weights, and the right four columns display the calculation weighted by the sample weights, denoted with W. We present the unweighted calculation as the fairness methods do not guarantee equalized false positive rates over the weighted data, but rather only on the unweighted---however, false positive rates are not equalized with either calculation method. 
\label{app:tab:fpr_budgeted}}
\end{table}

\begin{table}[]
\begin{tabular}{l|llll|llll}
\toprule
& \multicolumn{8}{c}{In-Process Fairness Method: True Positive Rates}\\

& \multicolumn{4}{c}{Unweighted} & \multicolumn{4}{c}{Weighted (W)}\\
\midrule
Income \\ Bucket  & Uconstr. & SR PAR & TPR PAR & EO    & Unconstr. W & SR PAR W & TRP Par W & EO W  \\
\toprule
0                     & 0.000         & 0.015   & 0.021    & 0.014 & 0.000         & 0.011     & 0.034      & 0.014 \\
7                     & 0.015         & 0.029   & 0.010    & 0.010 & 0.012         & 0.032     & 0.011      & 0.010 \\
13                    & 0.008         & 0.015   & 0.015    & 0.013 & 0.007         & 0.011     & 0.020      & 0.013 \\
18                    & 0.018         & 0.024   & 0.015    & 0.022 & 0.027         & 0.025     & 0.015      & 0.022 \\
26                    & 0.045         & 0.019   & 0.016    & 0.009 & 0.056         & 0.022     & 0.018      & 0.009 \\
36                    & 0.019         & 0.015   & 0.016    & 0.013 & 0.025         & 0.011     & 0.014      & 0.013 \\
47                    & 0.027         & 0.000   & 0.026    & 0.006 & 0.040         & 0.000     & 0.030      & 0.006 \\
62                    & 0.007         & 0.000   & 0.018    & 0.018 & 0.012         & 0.000     & 0.015      & 0.018 \\
86                    & 0.001         & 0.000   & 0.017    & 0.013 & 0.002         & 0.000     & 0.009      & 0.013 \\
126                   & 0.000         & 0.000   & 0.009    & 0.010 & 0.000         & 0.000     & 0.016      & 0.010 \\
\toprule
& \multicolumn{8}{c}{Post-Process Fairness Method: True Positive Rates}\\
\toprule
& \multicolumn{4}{c}{Unweighted} & \multicolumn{4}{c}{Weighted (W)}\\
\midrule
Income \\ Bucket  & Unconstr. & SR PAR & TPR PAR & EO    & Unconstr. W & SR PAR W & TRP Par W & EO W  \\
\toprule
0                     & 0.000         & 0.000   & 0.000    & 0.000 & 0.000         & 0.000     & 0.000      & 0.000 \\
7                     & 0.015         & 0.015   & 0.015    & 0.000 & 0.012         & 0.012     & 0.012      & 0.000 \\
13                    & 0.008         & 0.008   & 0.008    & 0.000 & 0.007         & 0.007     & 0.007      & 0.000 \\
18                    & 0.018         & 0.018   & 0.018    & 0.000 & 0.027         & 0.027     & 0.027      & 0.000 \\
26                    & 0.045         & 0.045   & 0.045    & 0.000 & 0.056         & 0.056     & 0.056      & 0.000 \\
36                    & 0.019         & 0.019   & 0.019    & 0.000 & 0.025         & 0.025     & 0.025      & 0.000 \\
47                    & 0.027         & 0.027   & 0.027    & 0.000 & 0.040         & 0.040     & 0.040      & 0.000 \\
62                    & 0.007         & 0.007   & 0.007    & 0.000 & 0.012         & 0.012     & 0.012      & 0.000 \\
86                    & 0.001         & 0.001   & 0.001    & 0.000 & 0.002         & 0.002     & 0.002      & 0.000 \\
126                   & 0.000         & 0.000   & 0.000    & 0.092 & 0.000         & 0.000     & 0.000      & 0.117\\
\bottomrule
\end{tabular}
\caption{We present the true positive rates by income bucket for the audit allocations generated from unconstrained and fairness-constrained random forest classifier models on the \emph{test} set, where an audit allocation corresponds to the highest ranked predictions from each model up to 0.644\% of the taxpayer population (i.e. around 1.1M audits). Unconstr. refers to an unconstrained model, SR PAR to selection rate parity, TPR PAR to true positive rate parity, and EO to equalized odds. Note that the only column where we would expect to see equalized true positive rates are the true positive rate parity (TPR PAR) equalized odds (EO) columns. The top table represents results from in-process fairness methods, and the lower table from post-process fairness enforcement methods. The numbers in the left side (left four columns) of the table corresponds to the calculation on the raw data, without sample weights, and the right four columns display the calculation weighted by the sample weights, denoted with W. We present the unweighted calculation as the fairness methods do not guarantee equalized true positive rates over the weighted data, but rather only on the unweighted---however, true positive rates are not equalized over income deciles in either calculation scheme. Income buckets are given in thousands.
\label{app:tab:tpr_budgeted}}
\end{table}

\begin{table}[]
\begin{tabular}{l|l|l|ll}
              & DP Enforc. & TPR Enforc. & EO Enforc. & EO Enforc. \\
              \toprule
Income Bucket & SRP        & TPR         & TPR        & FPR        \\\midrule
0             & 0.348      & 0.979       & 0.981      & 0.006      \\
7            & 0.348      & 0.981       & 0.980      & 0.009      \\
            & 0.349      & 0.983       & 0.982      & 0.013      \\
18            & 0.348      & 0.985       & 0.985      & 0.007      \\
26           & 0.367      & 0.986       & 0.986      & 0.006      \\
36           & 0.350      & 0.982       & 0.982      & 0.005      \\
47            & 0.368      & 0.993       & 0.993      & 0.004      \\
62            & 0.368      & 0.996       & 0.995      & 0.004      \\
86           & 0.368      & 0.996       & 0.996      & 0.004      \\
126          & 0.366      & 0.990       & 0.991      & 0.003     
\end{tabular}
\caption{We present a verification of the fact that in-process fairness techniques work as billed. From left to right, we have the selection rate by income bucket in the equalized selection rate model, the true positive rate by income bucket in the true positive parity constrained model, and the true and false positive rates by income bucket in the equalized odds constrained model. All results are presented over \emph{all predictions} in the \emph{training set}, not over an allocation the size of 0.644\% of taxpayer population (i.e. about 1.1M audits), as in the majority of the paper. This is in order to verify the guarantees the in-processing method implemented in \textit{FairLearn} actually provides, which is that the model will satisfy the fairness constraint desired \emph{in expectation on the training set}, within error $2(\epsilon$ + best\_gap), where best\_gap is a determined at run-time and not released to the model users, and $\epsilon$ is a user-set slack parameter. We set the slack parameter to 1\% in our implementation. Note that for each metric presented, all rates across income buckets are within 2\% of each other. Thus, the fairness metrics are satisfied within the expected parameters of
 $2(\epsilon)\leq$ $2(\epsilon $+ best\_gap). Income buckets are given in thousands. \label{app:tab:sanity_check}}

\end{table}

\subsection{Post-processing}
Post-processing involves intervening at prediction time by developing group-specific thresholds for positive predictions on top of the original model to ensure a model's predictions satisfy the relevant fairness constraints. We use a method developed by Hardt et al~\cite{hardtequality2016} to implement this technique.\\

\textit{Implementation.} In post-processing 
methods, the base random forest model is trained exactly as described in Section~\ref{app:exp_details}. We again use \textit{FairLearn}~\cite{bird2020fairlearn} to implement the post-processing technique based upon Hardt et al.~\cite{hardt2016equality}. Post-processing methods as implemented in \textit{FairLearn} are not engineered to return a ranking but only a binary prediction, thus in order to accommodate creating a ranking from predictions, we multiply the binary predictions of the fair classifier (which satisfy the desired metric across groups) by the predicted probabilities from the baseline classifier in order to be able to meaningfully rank the output. 

\paragraph{Results}
\begin{figure}
    \centering
    \includegraphics[scale=0.2]{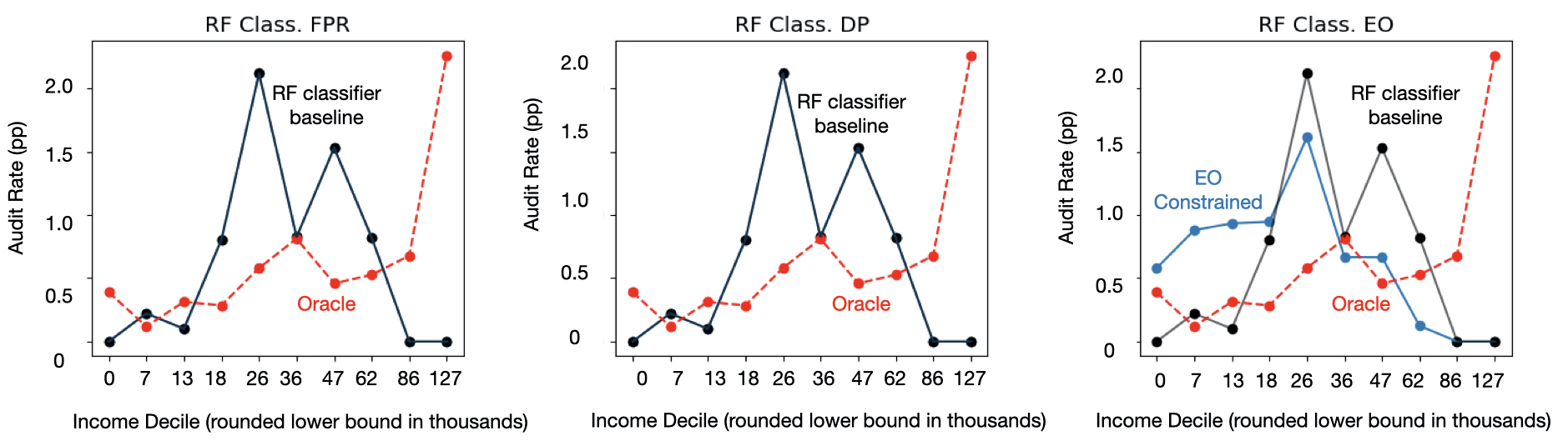}
    \caption{Post-process fairness techniques imposed on a random forest model. From left to right: enforcing Equal True Positive Rates (FP), Demographic Parity (DP), and Equalized Odds (EO). Each blue graph depicts of the results of enforcing a fairness constraint, the black graph is the original allocation.}
    \label{fig:post}
\end{figure}
Figure~\ref{fig:post} displays audit rate by income for post-processed Random Forest classifiers to respect each of the three fairness metrics. Again, the constrained model's audit rates are in blue, the unconstrained in black, and the oracle in red dashed. The revenue, no-change rate, and cost of each are also displayed in Table ~\ref{table:all_nums}. 

A key takeway is that {post-processing techniques are ill-fit to the audit allocation problem as they often result in minimal changes to prediction on the most confidently predicted points, which can leave aggregate audit allocations \emph{unchanged} from the unconstrained model}.
Figure \ref{fig:post} shows that the audit selection from post-processed models often lead to no change in aggregate audit rates (demographic parity, true positive rate parity). This is likely due to the fact that re-drawing group-specific thresholds to determine a final prediction which satisfies a fairness constraint is less likely to affect the most confidently predicted points, which we select for the top $0.644\%$. This is by design to keep error to a minimum, and to keep the post-processed model as similar to the original model as possible~\cite{hardt2016equality}.

In terms of the equalized odds allocations suggested by the post-processed random forest model, it is unclear what benefits enforcing these constraints provides, as they do not satisfy the respective fairness definitions on the top 0.644\% of predictions, as is noticable from the demographic parity allocation (which does not change from the baseline model). 
Additionally, enforcing equalized odds actually substantially increases audit focus on the lower end of the income distribution through this method, so we do not reduce audit focus on lower income individuals.

Thus, post-processing techniques are technically mismatched for the budgeted audit selection setting, and we argue, do not lead to an increase in equity.

\paragraph{Fair Ranking and Pre-Processing.}
We omit two major alternative categories of methods: \emph{pre-processing} and \emph{fair ranking}. Pre-processing methods alter the data before model training; this may be as simple as re-sampling the data or as involved as learning alternative representations of data that obfuscate any correlation between outcomes and sensitive features. Such methods tend to have sharp tradeoffs with accuracy~\cite{lamba2021empirical}, and often sacrifice interpretability, which may limit applicability in this setting. Fair ranking methods attempt to achieve fairness guarantees in settings where the \emph{ranking} of individuals matter.\cite{celis2017ranking}, \cite{singh2018fairness} While this may appear related to the audit problem, an important distinction is that in the fair ranking problem, the relative placement of items matters even beyond the decision to include or exclude them from some selection set. This is a more difficult setting than the audit problem as defined in Section \ref{sec:background}, in which the precise ranking \emph{within} audited taxpayers and separately \emph{within} non-audited taxpayers does not matter\footnote{ This may be less true if the \emph{budget} is not known in advance, but we do not consider such a scenario here.} to the IRS (nor does it matter to the taxpayers). Hence, methods aimed at fair ranking are `overkill' for our setting. 

\section{Revenue-Optimal Problem as Fractional Knapsack}
\label{app:cost_eqs}
Given audit variables $a_i$, net revenues $r_i$, costs $c_i$ and weights $w_i$, and a budget A, the revenue-optimal selection of audits is described by the following LP:
\begin{equation*}
\begin{array}{ll@{}ll}
\text{maximize}  & \displaystyle\sum\limits_{j=1}^{m} a_{j}&r^{net}_{j} \\
\text{subject to}&  \displaystyle\sum\limits_{j=1}^{m} a_{j}&c_{j} \leq A \\ 
                 &&a_{i} \in [0,w_i], &\forall  a_i
\end{array}
\end{equation*}

Note that this is simply an instantiation of the fractional knapsack problem, which is often intuitively described as, given an option of several items with different values and weights,
choosing a subset of $x$ items to put into a ``knapsack'' in order to maximize the value in the knapsack given the constraint of how much a person can carry (where, in the fractional approximation, one is allowed to put a fraction of the item in the knapsack). The analogues here is the audit allocation is our knapsack, taxpayers are items to put in the knapsack, total net revenue is the value, and the cost of each taxpayer audit to the IRS is the weight. The optimal solution to this problem is a greedy selection of the objects with the best value per unit weight, i.e., in our setting, taxpayers in order of the ratio of their net tax liability returned to the IRS over the cost to the IRS to audit that individual. 

\section{Cost Calculations}
\label{app:cost_cac}
We base our estimate of cost off of:\\
    (examiner time spent on an audit)*(cost per time unit of that grade examiner)\footnote{We note that this data recorded is grade of the lead examiner, but in some cases multiple people of different grades are involved. This is a shortcoming of the data for determining cost.} averaged over income decile and \emph{activity code} groups, which approximately corresponds to groupings of individuals based upon what tax forms they have filled out.
Importantly, we base our calculation of audit cost off of \emph{operational} IRS audits, i.e., not audits completed as a part of the National Research Program (NRP), but rather those conducted explicitly to enforce the tax code and reclaim misreported revenue. This is due to the fact that audits used for NRP are conducted differently, using more time-consuming methods, and thus relying on these cost estimates may provide a skewed picture of monetary cost to the IRS. We winsorize cost to 1st and 99th percentiles.
To calculate a dollar audit budget, we calculate the yearly cost of audits using our cost metrics from operational audit data from 2010-2014, and then we average this result by five to get the average dollar cost per year in amounts proportional to our conception of cost.
    \end{document}